\documentclass[11pt]{article}

\usepackage[final]{acl}

\usepackage[T1]{fontenc}
\usepackage[utf8]{inputenc}
\usepackage{newtxtext}
\usepackage{newtxmath}
\usepackage{inconsolata}

\usepackage{amsthm}
\usepackage{amsmath}
\usepackage{mathtools}
\usepackage{bm}

\AtBeginDocument{%
  \setlength{\abovedisplayskip}{7pt plus2pt minus5pt}%
  \setlength{\belowdisplayskip}{7pt plus2pt minus5pt}%
  \setlength{\abovedisplayshortskip}{0pt plus3pt}%
  \setlength{\belowdisplayshortskip}{4pt plus3pt minus3pt}%
}

\usepackage{booktabs}
\usepackage{multirow}
\usepackage{array}
\usepackage{tabularx}
\usepackage[table]{xcolor}
\usepackage{booktabs}
\usepackage{makecell}
\usepackage{siunitx}
\usepackage{graphicx}
\usepackage{subcaption}

\usepackage{algorithm}
\usepackage{algorithmic}

\usepackage{enumitem}
\usepackage{xspace}

\usepackage{xcolor}

\newcommand{\defeq}{\coloneqq}

\usepackage{hyperref}
\usepackage{cleveref}

\usepackage{amsthm}

\newtheorem{theorem}{Theorem}
\newtheorem{lemma}[theorem]{Lemma}

\newtheorem{corollary}[theorem]{Corollary}

\theoremstyle{definition}

\newtheorem{assumption}[theorem]{Assumption}

\theoremstyle{remark}

\usepackage{enumitem}

\newlist{assumlist}{enumerate}{1}
\setlist[assumlist]{
  label=(\roman*),
  leftmargin=*,
  itemsep=0.25em,
  topsep=0.25em
}

\usepackage[most]{tcolorbox}

\usepackage{titlesec}

\usepackage{float}
\title{HARP: Hadamard-Preconditioned Adaptive Rotation Processor \\for Extreme LLM Quantization}

\author{
  \textbf{Artur Zagitov\textsuperscript{1,2,3}} \quad
  \textbf{Gleb Molodtsov\textsuperscript{1,2}} \quad
  \textbf{Aleksandr Beznosikov\textsuperscript{1,2,3}} \\
 \textsuperscript{1}MIRAI \quad
  \textsuperscript{2}BRAIn Lab \quad
  \textsuperscript{3}Innopolis University
}

\begin{document}
\maketitle
\begin{abstract}
Post-training quantization (PTQ) is essential for deploying LLMs under memory and bandwidth constraints. However, extreme low-bit quantization remains highly sensitive to activation outliers and anisotropic weight curvature. Existing incoherence-based PTQ methods mitigate this issue with fixed randomized Hadamard transforms (RHTs), which improve quantization robustness but cannot adapt the rotated basis to the layer, calibration distribution, or quantizer. We introduce HARP (\textbf{H}adamard-preconditioned \textbf{A}daptive \textbf{R}otation \textbf{P}rocessor), a learnable structured two-sided orthogonal processor that replaces fixed Hadamard mixing while preserving exact full-precision equivalence.
HARP represents each rotation as a product of sparse butterfly-like block-orthogonal stages, supports non-power-of-two dimensions through Mixed-Radix schedules, and initializes to the RHT processor up to a fixed permutation. Fitted only on calibration data, HARP adapts the quantization basis to each layer and backend. Across 2--4-bit settings on Llama models from 1B to 70B, HARP consistently improves perplexity and yields its clearest zero-shot gains at 2 bits; a 2-bit Qwen3-8B experiment shows the same transfer beyond the Llama family. HARP also preserves deployment efficiency: on Llama~2 7B at 2 bits, it reaches 128 tok/s, retaining 90\% of RHT throughput (142 tok/s) and running approximately $2.1\times$ faster than FP16 (61 tok/s).
\end{abstract}

\section{Introduction}
\label{sec:intro}
Large language models (LLMs) have become central to modern NLP, yet their scale makes deployment increasingly constrained by memory bandwidth. During inference, repeated weight and activation transfers through the memory hierarchy dominate latency and serving cost. Quantization targets this bottleneck by compressing weights and activations to fewer bits.

Post-training quantization (PTQ) is especially attractive because it operates on a fixed pretrained model, uses only a small calibration set, and avoids the expense of full training \citep{gptq, smoothquant, awq, omniquant}.
Successful PTQ pipelines do more than choose a rounding rule. They are defined by several coupled design choices, including calibration, representation preprocessing, the reconstruction objective, and the code family used for compression.

In the extreme low-bit regime, outlier handling becomes central. At such bitwidths, small reconstruction errors are amplified across many layers, and heavy-tailed statistics make the quantization problem substantially harder.
As a result, the coordinates in which quantization is performed can fundamentally change the difficulty of the problem. 


Incoherence processing exploits this degree of freedom by applying structured orthogonal changes of basis before quantization. The full-precision layer remains unchanged, but weight mass and curvature-sensitive directions are spread across coordinates, making them less aligned with a few outlier axes \citep{quip}.
In practice, randomized Hadamard transforms (RHTs) have become the default preprocessing choice because they are exactly orthogonal, fast, and admit simple $\mathcal{O}(d \log d)$ kernels \citep{quip_sharp, qtip}.

Although the standard Hadamard/RHT processor provides a strong generic mixing basis, it does not adapt to the layer, the calibration distribution, or the block structure of the downstream quantizer.
Meanwhile, recent evidence shows that the choice of rotation is not incidental: end-to-end rotation methods reveal substantial variation across rotations and motivate learning them from data \citep{quarot, spinquant}, learnable butterfly rotations adapt structured orthogonal mixing to calibration data \citep{butterflyquant}, and closed-form data-aware transforms can improve over fixed Hadamard mixing under specific quantizer assumptions \citep{wush}.
These alternatives, however, often target different operating points, relax the drop-in structure of Hadamard-based PTQ, or introduce rotations that are too costly to use. 
This leaves a gap: the preprocessing module should be learnable from calibration data, remain an exact orthogonal change of basis, and preserve the Hadamard-like efficiency that makes incoherence processing practical.

We present HARP (\textbf{H}adamard-preconditioned \textbf{A}daptive \textbf{R}otation \textbf{P}rocessor), a learnable two-sided orthogonal processor for extreme low-bit PTQ.
HARP satisfies three requirements simultaneously: (i) adaptivity, by fitting rotations to each layer and quantizer backend; (ii) faithfulness, by remaining an exact change of basis that preserves the full-precision model; and (iii) efficiency, by using staged butterfly-like block-orthogonal factors \citep{butterflyfactorization} instead of dense rotations. 
Crucially, HARP is drop-in compatible with existing Hadamard-based PTQ pipelines: at initialization, it recovers the RHT processor up to a fixed permutation convention, after which calibration learns a structured orthogonal refinement around this strong baseline.

\subsection*{Contributions.}
Our main contributions are:
\begin{itemize}[nosep, leftmargin=4mm]
  \item We introduce HARP, a learnable two-sided orthogonal incoherence processor for PTQ that is fit using only calibration data and compatible with both scalar and vector-quantized backends.
  \item We design HARP as a drop-in replacement for fixed RHT/Hadamard preprocessing: at initialization, it recovers the randomized Hadamard processor up to a fixed permutation, and calibration learns a structured orthogonal refinement.
  \item We demonstrate consistent quality gains over fixed RHT across 2--4-bit settings on Llama models from 1B to 70B, transfer the same method to Qwen3-8B without backend retuning, and retain most of the RHT throughput advantage through hardware-aware kernels. Our code is publicly available\footnote{\url{https://github.com/brain-lab-research/HARP}}.
\end{itemize}

\noindent We next formalize the setting (Section~\ref{sec:method}), introduce the HARP processor (Section~\ref{sec:method:harp}), and evaluate it (Section~\ref{sec:experiments}). A detailed literature review is provided in Appendix~\ref{sec:related}.

\section{Setup}
\label{sec:method}
\paragraph{Layerwise PTQ objective.}
Modern PTQ pipelines typically cast quantization as a layerwise reconstruction problem.

We consider PTQ of a linear layer with
$W\in\mathbb{R}^{d_{\mathrm{out}}\times d_{\mathrm{in}}}$ and calibration inputs
$x\in\mathbb{R}^{d_{\mathrm{in}}}$.
As in Hessian-aware PTQ, we use the empirical second moment as a curvature proxy:
\begin{equation}
\label{eq:H_def}
H \;\defeq\; \mathbb{E}[x x^\top]
\in \mathbb{R}^{d_{\mathrm{in}} \times d_{\mathrm{in}}}.
\end{equation}

Given that, PTQ seeks an approximation $\widehat{W}$ constrained to a code family $\mathcal{Q}$ that minimizes a proxy for the output error. 
Scalar quantization rounds weights independently, which is simple and kernel-friendly but can be inefficient at very low bitwidths.
Vector quantization (VQ) instead quantizes small blocks jointly using a structured code family.
By shaping quantization noise in a higher-dimensional space, VQ can significantly reduce distortion at a fixed bitrate, but it introduces more complex encoding/decoding and raises hardware considerations for fast inference.
A common choice is a second-moment or Hessian-weighted quadratic objective, which captures how errors along different input directions affect the layer output. 
\begin{equation}
\label{eq:proxy_loss}
\mathcal{L}(W,\widehat W)
\defeq
\operatorname{Tr}\!\left(
(W-\widehat W)H(W-\widehat W)^\top
\right).
\end{equation}
This perspective yields scalable algorithms because it allows optimization independently per layer (and often per block within a layer), using tractable curvature approximations derived from calibration activations \citep{gptq, obs, obs2, i-obs, obc}.

\paragraph{Incoherence processing for PTQ.}

Extreme low-bit PTQ is often dominated by anisotropy and outliers:
a small subset of coordinates can carry disproportionate mass in $W$ and in the curvature proxy $H$.
This harms blockwise quantizers because their scales and code assignments must cover rare directions,
reducing effective resolution for the majority of weights.

Incoherence processing addresses this by applying an orthogonal change of basis so that energy is less concentrated in a few coordinates.
For a linear layer, we choose two orthogonal processors
$U\in O(d_{\mathrm{out}})$ and $V\in O(d_{\mathrm{in}})$, rotate the weight and curvature proxy as
\begin{equation*}
\begin{aligned}
\widetilde{W} &\defeq U^\top W V,
&
\widetilde{H} &\defeq V^\top H V .
\end{aligned}
\end{equation*}
The backend then quantizes the rotated weight, producing $\widehat{\widetilde{W}}$, and the quantized layer is mapped back to the original basis as
$\widehat{W}\defeq U\widehat{\widetilde{W}}V^\top$.
Because $U$ and $V$ are orthogonal, cyclicity of trace gives the invariant reconstruction objective
\begin{equation}
\label{eq:loss_invariance_new}
\begin{aligned}
\mathcal{L}(W,\widehat{W})
&=
\operatorname{Tr}\!\left(
(\widetilde{W}-\widehat{\widetilde{W}})
\,\widetilde{H}\,
(\widetilde{W}-\widehat{\widetilde{W}})^\top
\right).
\end{aligned}
\end{equation}
Thus, the full-precision layer and the target reconstruction objective are unchanged; only the basis exposed to the finite-code quantizer changes.

This is particularly important because low-bit quantizers are generally not invariant to rotations.
Scalar quantizers are axis-aligned, and blockwise VQ backends operate on fixed contiguous groups.
If large weights or curvature-sensitive directions align with a few coordinates, scales and code assignments must cover rare outliers, reducing resolution for the majority of values.
Randomized Hadamard transforms (RHTs) reduce this concentration generically.
HARP keeps the same orthogonal reparameterization principle, but learns a structured rotation that adapts to the layer and downstream blockwise quantizer.

Following QuIP, one can measure generic concentration using weight incoherence
\begin{equation}
\label{eq:weight_incoherence}
\mu_W(A) \defeq \sqrt{mn}\frac{\|A\|_\infty}{\|A\|_F},
\qquad A\in\mathbb{R}^{m\times n},
\end{equation}
and Hessian incoherence
\begin{equation}
\label{eq:hessian_incoherence}
\mu_H(H) \defeq \sqrt{n}\,\|Q\|_\infty,
\qquad H=Q\Lambda Q^\top .
\end{equation}
Lower values indicate that weights or curvature-sensitive directions are less concentrated in a few coordinate axes.
Classical incoherence scores are useful but not sufficient for modern blockwise quantizers.
A rotation can reduce outliers while still misaligning curvature with the quantizer's blocks, or conversely increase a generic Hessian incoherence score while producing lower backend-specific reconstruction error.
For this reason, we evaluate both classical metrics and quantizer-aligned diagnostics in Appendix~\ref{app:incoherence_diag}: pre-quantization weight incoherence, quantized-weight incoherence, off-block Hessian energy, diagonal Hessian-weighted distortion, and classical Hessian incoherence.
The diagnostics show that HARP does not simply optimize every generic incoherence score. Rather, it learns a basis that is more favorable for the deployed blockwise quantizer.

Fixed RHTs are a popular choice because they are orthogonal and fast.
However, they ignore per-layer statistics and quantizer structure.
HARP keeps Hadamard-like efficiency and exact orthogonality while learning a structured refinement from calibration data.

\section{HARP processors: structured orthogonal rotations}
\label{sec:method:harp}

A dense learned rotation would be too expensive to store or apply at LLM hidden dimensions.
HARP therefore parameterizes each processor as a product of sparse stride stages, following the same high-level structure as fast Hadamard and butterfly transforms.
Fix a transformed dimension $d$ (either $d_{\mathrm{in}}$ or $d_{\mathrm{out}}$) and choose a Mixed-Radix schedule
\begin{equation}
\label{eq:stage_schedule}
\begin{gathered}
\bm{b}=(b_0,\dots,b_{m-1}),
\qquad b_t \ge 2,\\
\prod_{t=0}^{m-1} b_t = d .
\end{gathered}
\end{equation}
The schedule determines which coordinates are mixed at each stage.
In all experiments we use a greedy schedule with preferred radix $8$ (unless otherwise noted). Appendix~\ref{app:greedy_schedule} gives the construction and Appendix~\ref{sec:experiments:radix_ablation} ablates the radix choice.

At stage $t$, define the stride
$s_t\defeq\prod_{j<t} b_j$ and the number of blocks
$D_t\defeq d/b_t$.
Conceptually, a perfect-shuffle permutation $P_t$ groups the $b_t$ coordinates mixed by stage $t$ into contiguous blocks.
The stage operator is
\begin{equation}
\label{eq:harp_stage}
\begin{aligned}
S_t(\Theta_t)
&\defeq
P_t^\top \mathcal{B}_t(\Theta_t)P_t,\\
\mathcal{B}_t(\Theta_t)
&\defeq
\mathrm{BlkDiag}\!\big(
B_{t,1},\ldots,B_{t,D_t}
\big),
\end{aligned}
\end{equation}
where each $B_{t,c}=B_{t,c}(\theta_{t,c})\in\mathbb{R}^{b_t\times b_t}$ is orthogonal.
The one-pass HARP transform is
\begin{equation}
\label{eq:harp_transform_def}
T(\Theta)
\defeq
S_{m-1}(\Theta_{m-1})\cdots S_0(\Theta_0).
\end{equation}
Since each stage is orthogonal, $T(\Theta)$ is orthogonal.
Multiple passes can be used by composing independent copies of the same schedule, although we use one pass by default.

Although Eq.~\eqref{eq:harp_stage} is written with a permutation, no explicit gather is used.
The input is reshaped to expose the stride-$s_t$ groups, transposed so that each length-$b_t$ group is contiguous, multiplied by the corresponding block kernel, and reshaped back.
This preserves the staged execution pattern that makes Hadamard preprocessing efficient.
Appendix~\ref{app:stride_impl_details} gives an index-view example of the stride layout.

HARP uses independent processors for the input and output sides, giving the rotations $V$ and $U$ used in Eq.~\eqref{eq:loss_invariance_new}.
As in RHT preprocessing, we also include diagonal Rademacher sign flips.
In our QuIP\# integration, these signs are already part of the pipeline. HARP reuses them and replaces only the fixed Hadamard mixing component.

\subsection{Hadamard-preconditioned block kernels and initialization}
\label{sec:method:kernels_init}

HARP uses block kernels of the form
\begin{equation}
\label{eq:block_def}
B_{t,c}(\theta_{t,c})
\defeq
Q_{t,c}(\theta_{t,c})\,G_{b_t},
\end{equation}
where $Q_{t,c}(\theta_{t,c})\in SO(b_t)$ is learnable and $G_{b_t}$ is a fixed orthogonal base mixer.
For power-of-two radices, $G_{b_t}$ is the normalized Sylvester Hadamard matrix $H_{b_t}$.
For non-power-of-two radices, we use a deterministic orthogonal fallback computed by QR factorization of a fixed Gaussian matrix.
Thus $G_{b_t}^\top G_{b_t}=I$, so every block remains orthogonal.

Hadamard preconditioning serves two purposes.
First, it anchors HARP to a strong baseline: at $\Theta=0$, all learnable blocks satisfy $Q_{t,c}(0)=I$, so $B_{t,c}(0)=G_{b_t}$.
For power-of-two schedules, the staged product therefore recovers Hadamard-family preprocessing up to a fixed permutation convention, so HARP is a drop-in upgrade for Hadamard-based incoherence processing pipelines.
Second, it makes fitting easier in practice: rather than learning mixing from scratch, HARP learns a structured residual around an already useful incoherence processor, which is helpful when the quantizer objective is highly non-smooth at extreme bitwidths.
Appendix~\ref{app:mixer_choice} ablates the base mixer choice.

\subsection{Exact initialization: \texorpdfstring{$\Theta=0$}{Theta=0} recovers RHT}
\label{sec:method:exact_init}

HARP initializes all learnable parameters to zero, so
$Q_{t,c}(0)=I_{b_t}$ and $B_{t,c}(0)=G_{b_t}$ for every stage $t$ and block $c$.
With Hadamard base mixers, the stride stages therefore recover Hadamard-family preprocessing exactly, up to the fixed permutation induced by the reshape convention.
Under the same Kronecker fallback used by QuIP\# for non-power-of-two dimensions, the initialization also matches QuIP\#'s convention.
Thus, before calibration, HARP is numerically equivalent to the existing RHT processor up to implementation-level floating-point differences.

\vspace{-2pt}
\begin{assumption}[Exact RHT equivalence conditions]
\label{assum:stride_equiv}
For each transformed dimension, the following conditions hold:
\begin{assumlist}[nosep,leftmargin=*]
\item For a transformed dimension $d$, either: (a) $d=2^L$ with a power-of-two stride schedule, or
(b) $d=K\cdot 2^L$ with the Kronecker fallback $T_d(\Theta)$ from Eq.~\eqref{eq:kron_full}.
\item For every power-of-two stage radix $b_t$, the base mixer is the normalized Sylvester Hadamard $G_{b_t}=H_{b_t}$.
\item Initialization is exact: $Q_{t,c}(0)=I_{b_t}$, hence $B_{t,c}(0)=G_{b_t}$.
\item Random sign flips use the same Rademacher signs and placement as the corresponding RHT processor:
$S_U\in\{\pm1\}^{d_{\mathrm{out}}}$ and
$S_V\in\{\pm1\}^{d_{\mathrm{in}}}$ are applied as diagonal matrices.
\end{assumlist}
\end{assumption}

\begin{theorem} \label{theorem}
Under Assumption~\ref{assum:stride_equiv}, HARP at $\Theta=0$ implements QuIP\#'s randomized Hadamard incoherence processor exactly up to the same fixed permutation convention. With the inherited random signs, the full two-sided processor is the corresponding RHT. Proof in Appendix~\ref{app:equiv}.
\end{theorem}

\subsection{Orthogonal parameterization}
\label{sec:method:param}

For $b_t=2$, we use a Givens rotation:
\begin{equation}
\label{eq:givens}
Q(\theta)=
\begin{bmatrix}
\cos\theta & -\sin\theta\\
\sin\theta & \cos\theta
\end{bmatrix}.
\end{equation}
For $b_t>2$, we use the Cayley map.
If $A(\theta)\in\mathbb{R}^{b_t\times b_t}$ is skew-symmetric, then
\begin{equation}
\label{eq:cayley_new}
Q(\theta)
\defeq
(I + A(\theta))^{-1}(I - A(\theta))
\end{equation}
is orthogonal whenever $I+A(\theta)$ is invertible.
In implementation, $A(\theta)$ is obtained by antisymmetrizing an unconstrained tensor.

\begingroup
\setlength{\textfloatsep}{2pt plus 1pt minus 1pt}
\subsection{Mixed-Radix dimensions and Kronecker fallback}
\label{sec:method:mixed_radix_kron}

The Mixed-Radix schedule in Eq.~\eqref{eq:stage_schedule} supports non-power-of-two dimensions directly.
For example, $5120$ is handled by the schedule $(8,8,8,5,2)$, using Hadamard mixers for the power-of-two stages and the QR fallback for the radix-$5$ stage.
This avoids padding while preserving exact orthogonality and staged execution.

We also support an optional Kronecker fallback that aligns the initialization more closely with QuIP\#'s Hadamard convention and reduces the number of learnable parameters.
If $d=K\cdot n_2$ with $n_2=2^L$, and a Hadamard-like table
$\widetilde{H}_K\in\{\pm1\}^{K\times K}$ satisfies
$\widetilde{H}_K\widetilde{H}_K^\top=KI$, we define
\begin{equation}
\label{eq:kron_full}
T_d(\Theta)
\defeq
\left(\tfrac{1}{\sqrt K}\widetilde{H}_K\right)
\otimes T_{n_2}(\Theta).
\end{equation}

Operationally, this reshapes a vector into a $K\times n_2$ tensor, applies HARP along the power-of-two axis, and mixes the $K$ rows with the fixed Hadamard-like table.
Mixed-Radix HARP is more general and expressive. The Kronecker fallback is useful when closer compatibility or lower overhead is desired.

\subsection{Implementation details}

\subsubsection{Memory and compute}
\label{sec:method:complexity}

A dense learned orthogonal matrix in $\mathbb{R}^{d\times d}$ requires $\Theta(d^2)$ storage and application cost, which is infeasible at LLM hidden dimensions.
HARP instead uses $m$ sparse block stages.
For a stage with radix $b_t$, there are $D_t=d/b_t$ blocks, each with $b_t(b_t-1)/2$ degrees of freedom, so the parameter count is
\begin{equation}
\label{eq:param_count}
D_t\frac{b_t(b_t-1)}{2}
=
d\frac{b_t-1}{2}.
\end{equation}
With bounded radices, this is $\mathcal{O}(dm)$ parameters per pass.
The application cost is $\mathcal{O}(D_tb_t^2)=\mathcal{O}(db_t)$ per stage and
$\mathcal{O}(\sum_t db_t)$ overall, i.e., Hadamard-like $\mathcal{O}(d\log d)$ for constant radices.
This gives learnable exact orthogonal mixing at structured-transform cost.

\begin{algorithm}[t]
\caption{HARP transform $y = x\,T(\Theta)$}
\label{alg:harp_transform}
\small
\begin{algorithmic}[1]
\REQUIRE Input $x\in\mathbb{R}^{B\times d}$, stage radices $(b_0,\dots,b_{m-1})$ with $\prod_t b_t=d$
\REQUIRE Fixed base mixers $G_{b_t}\in\mathbb{R}^{b_t\times b_t}$ and learnable parameters $\Theta=\{\theta_{t,c}\}$
\STATE $y \leftarrow x$
\STATE Precompute strides $s_t \defeq \prod_{r<t} b_r$ for $t=0,\dots,m-1$
\FOR{$t=0$ to $m-1$}
  \STATE $s \leftarrow s_t$, \quad $g \leftarrow d/(b_t s)$
  \STATE Reshape $y$ to $Y\in\mathbb{R}^{B\times g\times b_t\times s}$
  \STATE Transpose to $\widehat{Y}\in\mathbb{R}^{B\times g\times s\times b_t}$
  \STATE Construct block kernels $K_t\in\mathbb{R}^{(d/b_t)\times b_t\times b_t}$:
         $K_t[c] \leftarrow Q(\theta_{t,c})\,G_{b_t}$
  \STATE View $K_t$ as $\mathbb{R}^{g\times s\times b_t\times b_t}$
  \STATE Apply blockwise multiplication along the last dimension
  \STATE Invert transpose and reshape back to $y\in\mathbb{R}^{B\times d}$
\ENDFOR
\RETURN $y$
\end{algorithmic}
\end{algorithm}

\subsubsection{Parameter quantization for deployment}
\label{sec:method:theta_quant}

After fitting, we optionally store HARP parameters in int8 form.
For each block, we quantize either the Givens angle or the upper-triangular entries of the Cayley matrix using a per-block scale, and reconstruct the corresponding orthogonal block at runtime. This reduces processor storage with little effect on perplexity in our experiments. Appendix~\ref{app:theta_quant_details} gives the exact packing rule.

\subsubsection{Fitting HARP processors for PTQ}
\label{sec:method:fitting}

HARP is fit per layer using calibration statistics and a fixed quantizer backend.
Let
\begin{equation*}
\begin{gathered}
\widetilde{W}(\Theta)=U(\Theta_U)^\top W V(\Theta_V),\\
\widetilde{H}(\Theta)=V(\Theta_V)^\top H V(\Theta_V).
\end{gathered}
\end{equation*}

Directly optimizing the full proxy loss in Eq.~\eqref{eq:loss_invariance_new} inside the QuIP\# solver would be expensive: it would require repeated multiplication by $\widetilde{H}(\Theta)$ and repeated blockwise second-order quantization steps.
We therefore optimize a lightweight surrogate that is cheap enough to evaluate during calibration while remaining aligned with the deployed codebook quantizer.

QuIP\# ultimately refines blockwise code assignments using an $LDL^\top$ factorization of the curvature proxy.
During HARP fitting, we do not nest this LDLQ procedure inside the optimizer. Instead, for the current rotated weight we compute a direct blockwise codebook target and treat it as stopped-gradient. This gives a practical training signal: a useful rotation should make $\widetilde W$ easier for the downstream code family to represent, especially in curvature-important directions.

Let
$
\Delta(\Theta)
\defeq
\widetilde{W}(\Theta)-Q(\widetilde{W}(\Theta)),
\quad
w_j(\Theta)\defeq |\widetilde{H}_{jj}(\Theta)|,
$
and let $\bar w=w/\mathrm{mean}(w)$.
Our diagonal Hessian-weighted reconstruction proxy is
\begin{equation}
\label{eq:diag_proxy}
\mathcal{L}_{\mathrm{diag}}(\Theta)
\;\defeq\;
\frac{1}{d_{\mathrm{out}}d_{\mathrm{in}}}
\sum_{i=1}^{d_{\mathrm{out}}}\sum_{j=1}^{d_{\mathrm{in}}}
\Delta_{ij}(\Theta)^2\,\bar{w}_j .
\end{equation}
We stop gradients through $Q(\widetilde W)$, so gradients update the rotation parameters through $\widetilde W(\Theta)$ rather than through discrete code-assignment changes.
This is more stable than straight-through gradients through sharp codebook search.

To better align the rotated curvature proxy with blockwise quantization, we also penalize off-block Hessian energy:
\begin{equation}
\label{eq:hbd}
\mathcal{R}_{\mathrm{bd}}(\Theta_V)
\defeq
\frac{1}{d_{\mathrm{in}}^2}
\sum_{\substack{p,q\\p\neq q}}
\big\|\widetilde{H}_{pq}\big\|_F^2,
\end{equation}
where $\widetilde{H}_{pq}$ denotes a block under the quantizer's contiguous partition.
This term gives $V$ a direct signal to make curvature more local in the same block structure used by the backend.
The total fitting loss is
\begin{equation}
\label{eq:total_fit_loss}
\mathcal{L}_{\mathrm{fit}}(\Theta)
\defeq
\mathcal{L}_{\mathrm{diag}}(\Theta)
+
\lambda_{\mathrm{bd}}\mathcal{R}_{\mathrm{bd}}(\Theta_V).
\end{equation}
We optimize $\Theta_U,\Theta_V$ with Adam from the exact RHT initialization.
Appendix~\ref{app:harp_fit_algorithm} gives the full layerwise fitting procedure, including the target-refresh variant used to reduce calibration time.
\endgroup
\section{Experiments}
\label{sec:experiments}

\paragraph{Scope of comparison.}
Our primary experimental comparison isolates a single component: the incoherence processor inside a fixed QuIP\#-style backend.
We therefore compare HARP directly against RHT under the same backend, and separately report context-matched system-level comparisons to published AWQ, GPTQ, OmniQuant, ButterflyQuant, and the available weight-only SpinQuant result where evaluation protocols are compatible.
We also include a QTIP experiment to test whether the same learned processor can be inserted into a distinct RHT-based PTQ backend.


\paragraph{Models.}
The main scaling study uses Llama~3.2 (1B, 3B) and Llama~2 (7B, 13B, 70B). Subsection~\ref{sec:experiments:qwen3} reports a separate 2-bit transfer experiment on Qwen3-8B, using the same E8P/LDLQ backend and HARP optimization settings.

\vspace{-1pt}
\paragraph{No-finetuning comparison.}
Our main tables intentionally compare quantization \emph{without} QuIP\#'s weight fine-tuning stages.
This isolates the effect of replacing the fixed RHT mixer with HARP: the codebooks, solver, signs, calibration statistics, and evaluation harness are kept fixed.
This choice should not be read as an incompatibility with fine-tuning.
Appendix~\ref{app:ft_ablation} shows that HARP can be combined with QuIP\#'s fine-tuning-during-quantization stage, and that HARP with only this stage already improves over the RHT setting.

\vspace{-1pt}
\paragraph{Evaluation protocol.}
We report perplexity (PPL, $\downarrow$) on WikiText2 and C4. Llama~3.2 results use context length $8192$, while the main Llama~2 comparison uses context length $4096$, matching the corresponding QuIP\# calibration statistics. For comparison with published AWQ/GPTQ/OmniQuant/ButterflyQuant and available SpinQuant results, we additionally evaluate Llama~2 at context length $2048$ in Table~\ref{tab:baseline_2}. Zero-shot accuracy is evaluated with \texttt{lm\_eval} using \texttt{acc} on ARC-Challenge, ARC-Easy, PIQA, and WinoGrande. Evaluation is deterministic for a fixed checkpoint; Appendix~\ref{app:evaluation_uncertainty} reports benchmark standard errors for the principal 2-bit results.

\vspace{-1pt}
\paragraph{Methods compared.}
\emph{FP16} is the unquantized baseline. \emph{RHT} uses fixed randomized Hadamard preprocessing as in QuIP\#. Unless explicitly attached to another backend (e.g., QTIP+HARP), \emph{HARP} denotes the same QuIP\# backend with only its fixed Hadamard mixer replaced by learned HARP processors; random signs, codebooks, solver, and evaluation remain unchanged. The primary variant is Mixed-Radix HARP. The optional Kronecker compatibility variant is reported separately in Appendix~\ref{app:kron_ablation}.

\vspace{-1pt}
\paragraph{Effective bitrate accounting.}
We report effective bits-per-parameter (BPP) assuming the underlying quantized weights contribute exactly the nominal bitrate and adding storage overhead from HARP processor parameters and metadata.
RHT is therefore reported at exactly the nominal bitrate.
We report both floating-point HARP parameters and an int8 parameter-storage variant.

\begin{table*}[t]
\caption{Perplexity (PPL $\downarrow$) for QuIP\# with fixed RHT and Mixed-Radix HARP. Llama~3.2 uses context length $8192$ and Llama~2 uses $4096$. BPP includes HARP processor storage.}
\vspace{-8pt}
\label{tab:ppl_llama}
\centering
\scriptsize\sc
\setlength{\tabcolsep}{0.045cm}
\renewcommand{\arraystretch}{1.04}
\resizebox{\textwidth}{!}{
\begin{tabular}{@{}clccc ccc ccc ccc ccc@{}}
\toprule
& & \multicolumn{3}{c}{Llama~3.2 1B}
& \multicolumn{3}{c}{Llama~3.2 3B}
& \multicolumn{3}{c}{Llama~2 7B}
& \multicolumn{3}{c}{Llama~2 13B}
& \multicolumn{3}{c}{Llama~2 70B} \\
\cmidrule(lr){3-5}\cmidrule(lr){6-8}\cmidrule(lr){9-11}\cmidrule(lr){12-14}\cmidrule(lr){15-17}
Bits & Method
& BPP & W2 & C4
& BPP & W2 & C4
& BPP & W2 & C4
& BPP & W2 & C4
& BPP & W2 & C4 \\
\midrule
16 & FP16
& 16.00 & 11.57 & 13.19
& 16.00 & 9.58 & 10.61
& 16.00 & 5.12 & 6.63
& 16.00 & 4.57 & 6.05
& 16.00 & 3.12 & 4.97 \\
\midrule
2 & QuIP\# (RHT)
& 2.00 & 26.27 & 25.24
& 2.00 & 16.59 & 15.88
& 2.00 & 8.22 & 10.86
& 2.00 & 6.05 & 8.06
& 2.00 & 4.16 & 6.01 \\
\rowcolor[HTML]{F2F2F2}
2 & HARP
& 2.14 & \textbf{22.30} & \textbf{22.57}
& 2.10 & \textbf{15.02} & 14.77
& 2.11 & \textbf{7.23} & \textbf{9.49}
& 2.05 & \textbf{5.71} & \textbf{7.63}
& 2.04 & \textbf{4.01} & \textbf{5.81} \\
\rowcolor[HTML]{F2F2F2}
2 & \hspace{6pt} + int8 params
& 2.08 & 22.32 & \textbf{22.57}
& 2.05 & 15.03 & \textbf{14.76}
& 2.06 & 7.25 & 9.51
& 2.03 & 5.73 & 7.64
& 2.02 & \textbf{4.01} & 5.82 \\
\midrule
3 & QuIP\# (RHT)
& 3.00 & 14.02 & 15.32
& 3.00 & 11.04 & 11.72
& 3.00 & 5.61 & 7.35
& 3.00 & 4.89 & 6.49
& 3.00 & 3.40 & 5.20 \\
\rowcolor[HTML]{F2F2F2}
3 & HARP
& 3.14 & \textbf{13.46} & \textbf{14.56}
& 3.10 & \textbf{10.60} & \textbf{11.50}
& 3.11 & \textbf{5.51} & \textbf{7.19}
& 3.05 & \textbf{4.81} & \textbf{6.38}
& 3.04 & \textbf{3.35} & \textbf{5.15} \\
\rowcolor[HTML]{F2F2F2}
3 & \hspace{6pt} + int8 params
& 3.08 & 13.48 & 14.59
& 3.05 & 10.61 & 11.52
& 3.06 & 5.52 & 7.21
& 3.03 & 4.82 & 6.40
& 3.02 & \textbf{3.35} & 5.16 \\
\midrule
4 & QuIP\# (RHT)
& 4.00 & 12.37 & 13.95
& 4.00 & 10.13 & 11.04
& 4.00 & 5.27 & 6.85
& 4.00 & 4.70 & 6.20
& 4.00 & 3.22 & 5.05 \\
\rowcolor[HTML]{F2F2F2}
4 & HARP
& 4.14 & \textbf{12.17} & \textbf{13.66}
& 4.10 & \textbf{9.88} & \textbf{10.80}
& 4.11 & \textbf{5.22} & \textbf{6.78}
& 4.05 & \textbf{4.64} & \textbf{6.14}
& 4.04 & \textbf{3.19} & \textbf{5.02} \\
\rowcolor[HTML]{F2F2F2}
4 & \hspace{6pt} + int8 params
& 4.08 & 12.18 & 13.69
& 4.05 & 9.90 & 10.83
& 4.06 & 5.24 & 6.81
& 4.03 & 4.66 & 6.15
& 4.02 & 3.20 & \textbf{5.02} \\
\bottomrule
\end{tabular}}
\vspace{-4pt}
\end{table*}

\begin{figure*}[t]
\centering
\begin{subfigure}[t]{0.42\textwidth}
  \centering
  \includegraphics[width=\linewidth]{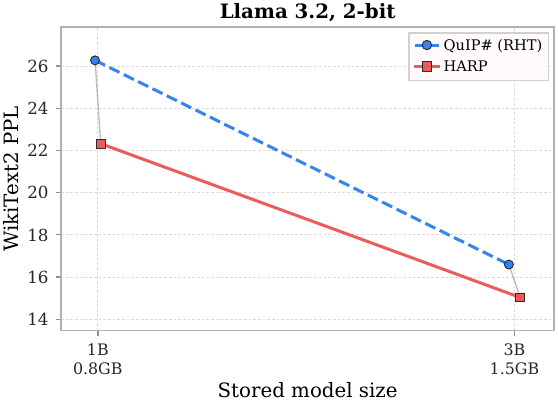}
  \label{fig:scaling_w2_2bit_small}
\end{subfigure}
\hspace{16pt}
\begin{subfigure}[t]{0.42\textwidth}
  \centering
  \includegraphics[width=\linewidth]{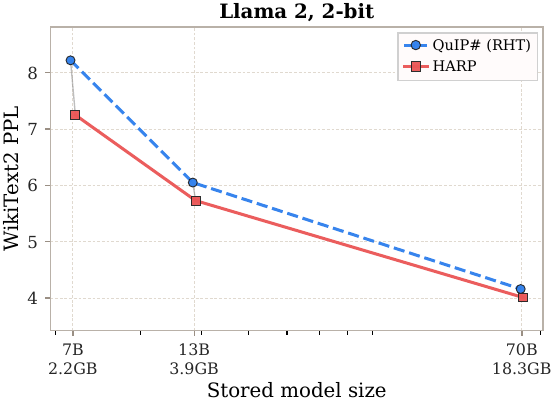}
  \label{fig:scaling_w2_2bit_large}
\end{subfigure}
\vspace{-16pt}
\caption{WikiText2 quality--size scaling at 2 bits. HARP uses int8 parameter storage. The plots are split by model family to keep the quality gaps visible across scales.}
\label{fig:scaling_w2_2bit}
\vspace{-12pt}
\end{figure*}

Figure~\ref{fig:scaling_w2_2bit} visualizes the 2-bit quality--size trade-off.
Across both model families, HARP shifts the RHT curve downwards at nearly the same stored model size when int8 parameter storage is used.
The improvement is greatest where the RHT model is farthest from FP16, while the gap naturally narrows for larger models whose 2-bit RHT baseline is already closer to full precision.
This suggests HARP is most effective when fixed incoherence processing leaves substantial quantization error.

\begin{table*}[t]
\caption{Zero-shot accuracy (\texttt{acc}, $\uparrow$) under QuIP\# (RHT) vs.\ HARP on Llama~2 using \texttt{lm\_eval}.
HARP uses Mixed-Radix. Context length follows the evaluation harness defaults.}
\vspace{-4pt}
\label{tab:zeroshot_llama2}
\centering
\scriptsize\sc
\setlength{\tabcolsep}{0.045cm}
\renewcommand{\arraystretch}{1.06}
\resizebox{\textwidth}{!}{
\begin{tabular}{@{}clccccc ccccc ccccc@{}}
\toprule
& & \multicolumn{5}{c}{Llama~2 7B}
& \multicolumn{5}{c}{Llama~2 13B}
& \multicolumn{5}{c}{Llama~2 70B} \\
\cmidrule(lr){3-7}
\cmidrule(lr){8-12}
\cmidrule(lr){13-17}
Bits & Method
& BPP & ArcC & ArcE & PIQA & Wino
& BPP & ArcC & ArcE & PIQA & Wino
& BPP & ArcC & ArcE & PIQA & Wino \\
\midrule
16 & FP16
& 16.00 & 40.0 & 69.3 & 78.5 & 67.3
& 16.00 & 45.6 & 73.3 & 78.7 & 69.6
& 16.00 & 51.1 & 77.7 & 81.1 & 77.0 \\
\midrule
2 & QuIP\# (RHT)
& 2.00 & 29.7 & 56.7 & 70.8 & 62.1
& 2.00 & 33.8 & 65.1 & 74.4 & 64.3
& 2.00 & 47.4 & 76.9 & 79.5 & 75.0 \\
\rowcolor[HTML]{F2F2F2}
2 & HARP
& 2.11 & \textbf{33.0} & \textbf{63.7} & \textbf{72.4} & \textbf{62.9}
& 2.05 & \textbf{36.4} & \textbf{67.3} & \textbf{75.8} & \textbf{67.2}
& 2.04 & \textbf{48.5} & \textbf{77.8} & \textbf{79.9} & \textbf{75.5} \\
\midrule
3 & QuIP\# (RHT)
& 3.00 & 37.7 & \textbf{69.3} & 76.2 & 66.7
& 3.00 & 41.4 & 71.6 & 77.5 & 68.0
& 3.00 & 50.4 & 78.4 & \textbf{80.7} & \textbf{77.0} \\
\rowcolor[HTML]{F2F2F2}
3 & HARP
& 3.11 & \textbf{38.0} & 69.1 & \textbf{77.1} & \textbf{67.3}
& 3.05 & \textbf{42.3} & \textbf{71.9} & \textbf{78.5} & \textbf{68.4}
& 3.04 & \textbf{51.8} & \textbf{78.6} & 80.6 & 76.8 \\
\midrule
4 & QuIP\# (RHT)
& 4.00 & \textbf{40.7} & \textbf{70.6} & 77.2 & 66.7
& 4.00 & \textbf{45.6} & 74.5 & 78.6 & \textbf{69.4}
& 4.00 & 51.5 & 78.1 & 80.6 & 78.1 \\
\rowcolor[HTML]{F2F2F2}
4 & HARP
& 4.11 & \textbf{40.7} & 69.2 & \textbf{78.3} & \textbf{68.7}
& 4.05 & 45.0 & \textbf{74.8} & \textbf{78.9} & \textbf{69.4}
& 4.04 & \textbf{51.9} & \textbf{78.2} & \textbf{81.2} & \textbf{78.4} \\
\bottomrule
\end{tabular}}
\vspace{-6pt}
\end{table*}

\subsection{Language modeling perplexity}
\label{sec:experiments:ppl}

Table~\ref{tab:ppl_llama} reports the controlled comparison between fixed RHT and Mixed-Radix HARP. Replacing only the mixer lowers perplexity for every evaluated model and bitwidth while preserving the codebooks, random signs, Hessians, and QuIP\# solver. The largest gains occur at 2 bits, where scalar and fixed-RHT quantizers leave the most reconstruction error. Figure~\ref{fig:scaling_w2_2bit} shows that these gains persist at nearly matched stored size after int8 parameter packing.

At 3 and 4 bits, RHT is already close to FP16, so the remaining absolute improvement is necessarily smaller. We therefore view these settings as evidence that HARP remains useful as quantization becomes less constrained, while the main operating point is extreme 2-bit PTQ. Appendix~\ref{app:baseline_2048} provides context-matched comparisons with published scalar and rotation baselines.

\paragraph{Kronecker compatibility and parameter storage.}
The optional Kronecker construction follows QuIP\#'s non-power-of-two convention with lower processor overhead, but is less expressive than the primary Mixed-Radix processor. Its available 1B--7B results are reported separately in Appendix~\ref{app:kron_ablation}. Int8 storage changes perplexity only marginally and reduces the HARP premium to 0.02--0.08 BPP in the 2-bit experiments.

\subsection{Quantizer-aligned structural analysis}
\label{sec:experiments:mechanism}

To examine why the learned basis improves quantization, we dequantize the E8P indices for all 128 linear matrices of 2-bit Llama~2 7B and evaluate $E=\widehat W-W$ against each calibration Hessian. Table~\ref{tab:reconstruction_summary} reports the QuIP\# proxy $\operatorname{tr}(EHE^\top)/\operatorname{tr}(WHW^\top)$, relative Frobenius error, and gain-error component $c=\langle E,W\rangle/\lVert W\rVert_F^2$, where $c<0$ indicates systematic shrinkage.

\begin{table}[t]
\centering
\caption{Aggregate reconstruction diagnostics over 128 quantized matrices of Llama~2 7B at 2 bits.}
\label{tab:reconstruction_summary}
\small
\setlength{\tabcolsep}{0.09cm}
\begin{tabular}{@{}lrrrr@{}}
\toprule
Processor & Proxy & Rel. Frobenius & $\lvert c\rvert$ & Proxy wins \\
\midrule
RHT  & 0.0413 & 0.1288 & 0.0573 & --- \\
HARP & \textbf{0.0375} & \textbf{0.1146} & \textbf{0.0451} & \textbf{123/128} \\
\bottomrule
\end{tabular}
\end{table}

HARP reduces mean proxy error by $9.2\%$, relative Frobenius error by $11.0\%$, and systematic shrinkage by $21.3\%$. The improvement is not explained solely by scalar flattening. E8P quantizes contiguous 8-dimensional blocks with a bounded codeword radius. Blocks outside this representable shell incur unavoidable radial distortion. Figure~\ref{fig:e8p_alignment_main} shows that HARP moves source-block norms toward the shell and lowers distortion across the norm range. For the layer-0 output projection, mean block distortion falls from $6.90$ to $2.50$, while the fraction of blocks above the maximum codeword norm falls from $39.0\%$ to $11.0\%$. Extended diagnostics are reported in Appendix~\ref{app:incoherence_diag}. The regularizer ablation in Appendix~\ref{app:hbd_ablation} shows that mild off-block regularization improves quality, while most of the gain remains without this term.

\begin{figure*}[t]
\centering
\includegraphics[width=0.91\textwidth]{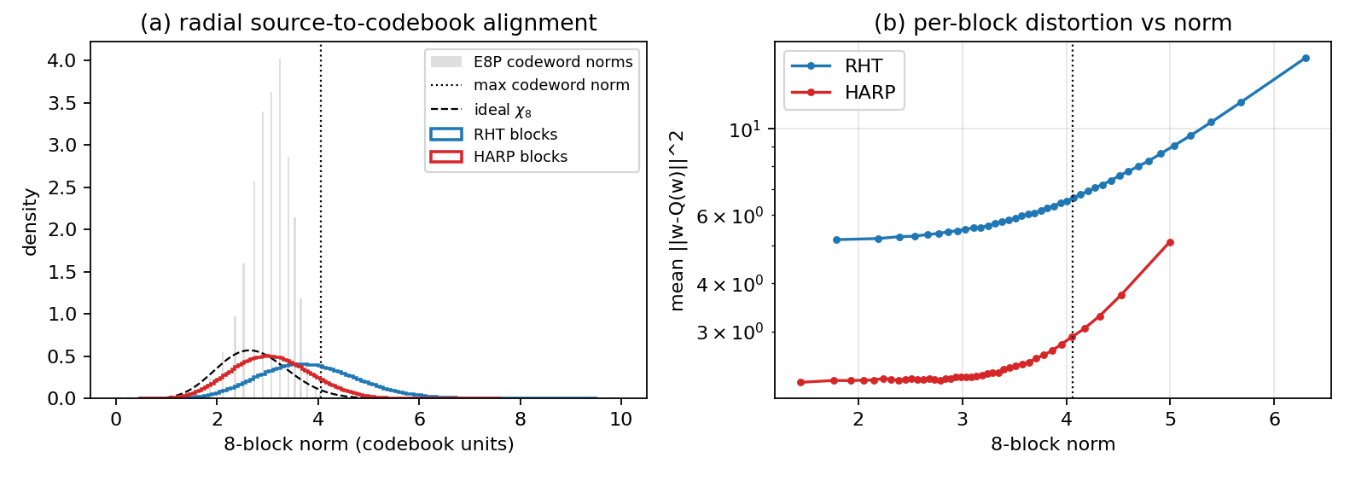}
\caption{E8P codebook alignment for the layer-0 output projection of 2-bit Llama~2 7B under fixed RHT versus HARP. The dotted vertical line marks the
maximum E8P codeword norm.}
\vspace{-2mm}
\label{fig:e8p_alignment_main}
\end{figure*}

\subsection{Zero-shot evaluation}
\label{sec:experiments:zeroshot}

Table~\ref{tab:zeroshot_llama2} reports zero-shot accuracy on Llama~2. At 2 bits, HARP improves every listed task over fixed RHT, with the largest gains on ARC. These results show that the perplexity improvements are accompanied by stronger downstream multiple-choice performance. At 3--4 bits, several task-level differences are small relative to benchmark standard error, consistent with the narrower perplexity gaps in this regime. We therefore emphasize the consistent 2-bit improvements rather than claiming uniform task-level gains at every bitrate. Appendix~\ref{app:evaluation_uncertainty} reports standard errors for the principal 2-bit results.

\subsection{Inference latency and storage}
\label{sec:experiments:latency}

A central motivation for HARP is to improve incoherence processing without giving up the speed advantages of a quantized backend.
We measure single-token latency using the QuIP\# decoding-style microbenchmark: batch size $1$, sequence length $1$, one warmup pass, and $2000$ timed repetitions, with CUDA graphs and SDPA attention enabled.
CUDA graph capture avoids repeated small-kernel launch overhead.
All timings in Table~\ref{tab:latency} use an NVIDIA RTX 5080 with 16GB VRAM.

\begin{table}[t]
\centering
\caption{Single-token latency on RTX 5080.
HARP uses Mixed-Radix and int8 parameter storage.}
\vspace{-4pt}
\label{tab:latency}
\small\sc
\setlength{\tabcolsep}{0.08cm}
\renewcommand{\arraystretch}{1.06}
\begin{tabular}{@{}llccc@{}}
\toprule
Model & Method & Bits & sec/token $\downarrow$ & tok/s $\uparrow$ \\
\midrule
Llama~2 7B & FP16 & 16 & 0.0162 & 61 \\
& QuIP\# (RHT) & 2 & 0.0070 & 142 \\
\rowcolor[HTML]{F2F2F2}
& HARP & 2 & 0.0078 & 128 \\
& QuIP\# (RHT) & 4 & 0.0100 & 100 \\
\rowcolor[HTML]{F2F2F2}
& HARP & 4 & 0.0110 & 91 \\
\midrule
Llama~2 13B & FP16 & 16 & OOM & OOM \\
& QuIP\# (RHT) & 2 & 0.0110 & 91 \\
\rowcolor[HTML]{F2F2F2}
& HARP & 2 & 0.0119 & 84 \\
& QuIP\# (RHT) & 4 & 0.0152 & 66 \\
\rowcolor[HTML]{F2F2F2}
& HARP & 4 & 0.0167 & 60 \\
\bottomrule
\end{tabular}
\end{table}

HARP preserves most of the fixed-RHT throughput while improving quantized-model quality. On Llama~2 7B at 2 bits, HARP reaches 128 tok/s versus 142 tok/s for RHT, retaining 90.1\% of RHT throughput, and is approximately $2.1\times$ faster than FP16 at 61 tok/s. The remaining overhead comes mainly from the additional stride-stage rotations. We use a fused Triton kernel for each stage, avoiding materialized permutations and reducing latency by approximately $20\%$ relative to an unfused implementation. Calibration is a one-time model-preparation cost and is analyzed separately in Appendix~\ref{app:calibration_time}; int8 parameter packing keeps the stored-model overhead small (Appendix~\ref{app:model_size}).

\subsection{Backend portability: QTIP}
\label{sec:experiments:qtip}

We also integrate the same HARP module into QTIP, which uses trellis-coded quantization together with incoherence processing \citep{qtip}.
The integration changes only the preprocessing, the QTIP backend remains fixed.
Table~\ref{tab:qtip_pilot} reports Llama~2 results at 2 and 3 bits.

\begin{table}[ht]
\centering
\caption{Perplexity (PPL $\downarrow$) for QTIP with fixed RHT and HARP.
HARP uses Mixed-Radix.}
\vspace{-4pt}
\label{tab:qtip_pilot}
\small\sc
\setlength{\tabcolsep}{0.10cm}
\renewcommand{\arraystretch}{1.06}
\begin{tabular}{@{}clccc cc@{}}
\toprule
 & & & \multicolumn{2}{c}{Llama~2 7B}
& \multicolumn{2}{c}{Llama~2 13B} \\
\cmidrule(lr){4-5}
\cmidrule(lr){6-7}
Bits & Method
&$\sim$ BPP & W2 & C4
& W2 & C4 \\
\midrule
2 & QTIP (RHT)
& 2.00 & 6.87 & 9.00
& 5.64 & 7.46 \\
\rowcolor[HTML]{F2F2F2}
2 & QTIP + HARP
& 2.10 & \textbf{6.62} & \textbf{8.79}
& \textbf{5.51} & \textbf{7.34} \\
\rowcolor[HTML]{F2F2F2}
2 & \hspace{6pt} + int8
& 2.05 & 6.63 & 8.91
& 5.52 & 7.36 \\
\midrule
3 & QTIP (RHT)
& 3.00 & 5.41 & 7.03
& 4.75 & 6.31 \\
\rowcolor[HTML]{F2F2F2}
3 & QTIP + HARP
& 3.10 & \textbf{5.37} & \textbf{6.98}
& \textbf{4.74} & \textbf{6.26} \\
\rowcolor[HTML]{F2F2F2}
3 & \hspace{6pt} + int8
& 3.05 & \textbf{5.37} & 6.99
& \textbf{4.74} & 6.27 \\
\bottomrule
\end{tabular}
\end{table}

HARP improves QTIP at every evaluated 2- and 3-bit setting, providing evidence that the learned processor is not specific to the QuIP\# code family. The int8 variant preserves nearly all gains. Because QTIP quantizes $Q$, $K$, and $V$ separately, their processors are also fit separately. Sharing the input-side processor across these projections is a possible route to lower calibration cost.

\vspace{-2pt}
\subsection{Comparison with published baselines}

Many published weight-only baselines report Llama~2 perplexity at context length $2048$, whereas the controlled main comparison uses $4096$. Table~\ref{tab:baseline_2} therefore reports a separate context-matched system-level comparison. HARP gives the lowest perplexity in all four 2-bit cells. The ButterflyQuant row uses its published scalar W2A16 pipeline, and the other baselines likewise use different quantizers and metadata. This table compares complete systems, while the fixed-backend RHT/HARP experiment isolates the processor.

\begin{table}[htbp]
\centering
\caption{Context-matched 2-bit Llama~2 perplexity at context length $2048$. Effective BPP is shown where directly comparable.}
\vspace{-4pt}
\label{tab:baseline_2}
\small
\setlength{\tabcolsep}{0.06cm}
\renewcommand{\arraystretch}{1.06}
\resizebox{\columnwidth}{!}{
\begin{tabular}{@{}llccccc@{}}
\toprule
Method & Backend & \makecell{BPP\\7B/13B} & 7B W2 & 7B C4 & 13B W2 & 13B C4 \\
\midrule
ButterflyQuant & scalar W2A16 & --- & 15.40 & 16.61 & 10.24 & 12.48 \\
AWQ & scalar, g128 & 2.14 & $2.2{\times}10^{5}$ & $1.7{\times}10^{5}$ & $1.2{\times}10^{5}$ & $9.4{\times}10^{4}$ \\
GPTQ & scalar, g128 & 2.14 & 36.77 & 33.70 & 28.14 & 20.97 \\
OmniQuant & scalar, g128 & 2.14 & 11.06 & 15.02 & 8.26 & 11.05 \\
QuIP\# (RHT) & E8P/LDLQ & 2.00/2.00 & 8.95 & 11.22 & 6.52 & 8.32 \\
\rowcolor[HTML]{F2F2F2}
HARP & E8P/LDLQ & 2.11/2.05 & \textbf{7.85} & \textbf{9.68} & \textbf{6.13} & \textbf{7.87} \\
\bottomrule
\end{tabular}}
\vspace{-4.5mm}
\end{table}

\subsection{Transfer to Qwen3-8B}
\label{sec:experiments:qwen3}

We additionally evaluate Qwen3-8B at 2 bits and context length
$4096$. The integration changes only model-specific module mapping
and requires fresh second-moment statistics. The E8P/LDLQ backend,
Mixed-Radix processor, fitting objective, and optimization
hyperparameters remain unchanged. Table~\ref{tab:qwen3_main} shows
that HARP improves perplexity on both corpora, extending the
controlled RHT-to-HARP gain beyond the Llama family.

\begin{table}[t]
\centering
\caption{Two-bit Qwen3-8B perplexity at context length $4096$
(mean $\pm$ evaluation standard error).}
\label{tab:qwen3_main}
\small
\setlength{\tabcolsep}{0.10cm}
\begin{tabular}{@{}lcc@{}}
\toprule
Method & WikiText2 $\downarrow$ & C4 $\downarrow$ \\
\midrule
FP16 & $8.99\pm0.27$ & $12.48\pm0.58$ \\
QuIP\# + RHT & $11.38\pm0.35$ & $15.30\pm0.74$ \\
\rowcolor[HTML]{F2F2F2}
QuIP\# + HARP & $\mathbf{10.80\pm0.33}$ &
$\mathbf{14.71\pm0.72}$ \\
\bottomrule
\end{tabular}
\end{table}

\subsection{Calibration scalability}
\label{sec:experiments:calibration}

Table~\ref{tab:calibration_70b_main} summarizes the 70B
preparation cost. Each quantized module group receives 1200
updates. The refresh interval controls how frequently the
stopped-gradient codebook target is recomputed.

\begin{table}[H]
\centering
\caption{Calibration and quantization cost for 2-bit
Llama~2 70B on H100. GPU-hours exclude Hessian generation.}
\label{tab:calibration_70b_main}
\small
\setlength{\tabcolsep}{0.10cm}
\begin{tabular}{@{}lrrr@{}}
\toprule
Method & Refresh $k$ & GPU-hours & W2 PPL $\downarrow$ \\
\midrule
RHT  & --- & $\sim 6$  & 4.16 \\
HARP & 1   & $\sim 80$ & \textbf{4.01} \\
HARP & 2   & $\sim 48$ & 4.03 \\
HARP & 4   & $\sim 31$ & 4.07 \\
\bottomrule
\end{tabular}
\vspace{-2mm}
\end{table}

Target caching reduces HARP preparation cost by up to $61\%$
while retaining an improvement over fixed RHT. Peak HARP
memory is approximately 42--50 GB, so each module group fits
on one 80 GB H100. Independent Transformer blocks can also be
scheduled across GPUs. Appendix~\ref{app:calibration_time}
reports the complete cross-scale results and Hessian-generation
accounting.

\section{Conclusion}
\label{sec:conclusion}

We introduced HARP, a learnable structured orthogonal incoherence processor for low-bit PTQ. HARP initializes exactly to fixed RHT and learns a layer- and backend-aware refinement from calibration data. Mixed-Radix butterfly-like stages support practical transformer dimensions at controlled storage and inference overhead. Across Llama~3.2 and Llama~2 at 2--4 bits, HARP consistently improves over fixed RHT under an otherwise unchanged QuIP\# backend, with the largest gains at 2 bits. The same configuration also improves Qwen3-8B and transfers to QTIP. Structural diagnostics indicate that the learned basis improves alignment with the deployed E8P codebook and blockwise curvature rather than merely optimizing generic incoherence. Int8 parameter storage preserves nearly all gains, and fused kernels retain most of the RHT throughput advantage.

\section*{Limitations}
HARP adds a one-time layerwise fitting stage beyond fixed RHT. Appendix~\ref{app:calibration_time} reports this cost through Llama~2 70B: depending on target-refresh frequency, the 70B fit requires approximately 31--80 GPU-hours and 42--50 GB peak VRAM on H100 hardware. Target caching, chunk-size control, and layer-level parallelism provide direct cost--memory trade-offs. Crucially, no end-to-end model fine-tuning is required.

The evaluated QuIP\# and QTIP backends are weight-only systems, so the controlled experiments isolate weight-only incoherence processing. Extending HARP to weight--activation and KV-cache quantization requires graph-level placement of rotations and coordination with residual streams, rotary embeddings, cache formats, and fused attention kernels.

The current code families expose 2--4-bit operating points. Applying HARP to binary or sub-one-bit quantizers requires a different code family, scaling rule, and quantizer-aligned objective. Finally, the structural analysis identifies codebook alignment and blockwise curvature as important mechanisms, but a complete account of residual-error directions and cross-layer accumulation remains future work.




\bibliography{arxiv}

\appendix
\allowdisplaybreaks


\section{Hadamard stride factorization and initialized equivalence}
\label{app:equiv}

\paragraph{Goal.}
We prove that HARP at $\Theta=0$ recovers Hadamard-family incoherence processing under the conventions used in the main text.
We show: (i) the Walsh--Hadamard transform admits a stride-stage factorization, (ii) initialized HARP implements this factorization up to a fixed permutation, and (iii) the Kronecker fallback matches the corresponding QuIP\# convention.

\paragraph{Conventions.}
We work with column vectors in this appendix for clarity.
The row-vector convention used in the main text is the transpose of these statements.
For $b=2^k$, let the unnormalized Sylvester matrices be defined by
\[
\widetilde H_1=[1],
\qquad
\widetilde H_{2n}=
\begin{bmatrix}
\widetilde H_n & \widetilde H_n\\
\widetilde H_n & -\widetilde H_n
\end{bmatrix},
\]
and define $H_b\defeq b^{-1/2}\widetilde H_b$.
Then $H_b^\top H_b=I_b$.

\subsection{Walsh--Hadamard as stride stages}

Assume $d=b^m$ with $b=2^k$.
For stage $t\in\{0,\dots,m-1\}$, define
\[
s_t\defeq b^t,
\qquad
g_t\defeq \frac{d}{bs_t}=b^{m-t-1},
\]
and the stride stage
\[
S_t\defeq I_{g_t}\otimes H_b\otimes I_{s_t}.
\]

\begin{lemma}[Stride implementation equals the Kronecker stage]
\label{lem:stride_equals_kron}
Let $x\in\mathbb{R}^d$.
Reshape $x$ into $X\in\mathbb{R}^{g_t\times b\times s_t}$ by
\[
X[\alpha,r,\beta]
\defeq
x[\alpha(bs_t)+rs_t+\beta].
\]
Apply $H_b$ along the middle index $r$ independently for each $(\alpha,\beta)$, and flatten back using the same indexing convention.
The resulting vector is $S_tx$.
\end{lemma}

\begin{proof}
For fixed $(\alpha,\beta)$, the operation multiplies the length-$b$ vector over index $r$ by $H_b$ and leaves all other indices unchanged.
This is exactly the action of $I_{g_t}\otimes H_b\otimes I_{s_t}$ under the stated flattening order.
\end{proof}

\noindent Define
\[
\mathcal H_d^{\mathrm{stride}}
\defeq
S_{m-1}\cdots S_0.
\]

\begin{lemma}[Stride product is a Hadamard transform up to permutation]
\label{lem:stride_product_is_wht}
There exists a fixed permutation matrix $P$ depending only on the digit/reshape convention such that
\[
\mathcal H_d^{\mathrm{stride}}
=
P^\top H_d P .
\]
\end{lemma}

\begin{proof}
Write each index as base-$b$ digits,
$i=\sum_{\ell=0}^{m-1} i_\ell b^\ell$.
Stage $S_t$ mixes digit $i_t$ and holds all other digits fixed.
The product therefore applies $H_b$ once per digit, i.e., it is $H_b^{\otimes m}$ up to a digit-axis permutation.
That permutation is represented by $P$.
Under the Sylvester convention, $H_d=H_b^{\otimes m}$.
\end{proof}

\subsection{Initialized HARP equals the stride Hadamard mixer}

\begin{theorem}[Exact equivalence at $\Theta=0$, restated as Thm. \ref{theorem}]
\label{thm:init_equiv_pow2}
Assume $d=b^m$ with $b=2^k$, and HARP uses $G_b=H_b$ for every stage.
At initialization, $Q_{t,c}(0)=I$ for all stages and blocks.
Then
\[
T(0)=\mathcal H_d^{\mathrm{stride}}
=
P^\top H_d P .
\]
\end{theorem}

\begin{proof}
At initialization, every HARP block is
$B_{t,c}(0)=Q_{t,c}(0)G_b=H_b$.
Thus each HARP stage is exactly the stride stage $S_t$ from Lemma~\ref{lem:stride_equals_kron}.
Taking the product over stages gives
$T(0)=S_{m-1}\cdots S_0=\mathcal H_d^{\mathrm{stride}}$.
The final equality follows from Lemma~\ref{lem:stride_product_is_wht}.
\end{proof}

\subsection{Kronecker fallback}

Assume $d=K\cdot 2^L$ and $\widetilde H_K\in\{\pm1\}^{K\times K}$ satisfies
$\widetilde H_K\widetilde H_K^\top=KI_K$.
Let $H_K\defeq \widetilde H_K/\sqrt K$.
The Kronecker fallback defines
\[
T_d(\Theta)=H_K\otimes T_{2^L}(\Theta).
\]

\begin{corollary}[Exact Kronecker equivalence at $\Theta=0$]
\label{cor:kron_init_equiv}
If $T_{2^L}(0)=P^\top H_{2^L}P$, then
\[
T_d(0)
=
(I_K\otimes P)^\top
(H_K\otimes H_{2^L})
(I_K\otimes P).
\]
Thus the initialization matches the Kronecker--Hadamard preprocessing convention up to the same fixed permutation on the power-of-two axis.
\end{corollary}

\begin{proof}
Substitute $T_{2^L}(0)=P^\top H_{2^L}P$ and use the mixed-product property of Kronecker products.
\end{proof}

\section{Comparison to published PTQ baselines}
\label{app:baseline_2048}

The main Llama~2 results use context length $4096$, while most published AWQ, GPTQ, OmniQuant, ButterflyQuant, and available SpinQuant weight-only values use $2048$. Table~\ref{tab:baseline_2048_full} therefore reports a separate context-matched comparison. HARP is strongest at 2 and 3 bits. At 4 bits, all competitive systems are close to FP16 and the differences are correspondingly small. ButterflyQuant uses scalar W2A16, while HARP uses E8P/LDLQ. SpinQuant's row is its available weight-only ablation rather than its main weight--activation configuration. These rows are system-level references across different quantizers and calibration procedures. The isolated processor comparison remains fixed RHT versus HARP under the identical QuIP\# backend.

\begin{table*}[!htbp]
\centering
\caption{Context-matched weight-only PTQ comparison on Llama~2 at context length $2048$. Effective BPP is reported for 7B/13B.}
\label{tab:baseline_2048_full}
\small
\setlength{\tabcolsep}{0.075cm}
\renewcommand{\arraystretch}{1.06}
\begin{tabular}{@{}clccccc@{}}
\toprule
Bits & Method & \makecell{Eff. BPP\\7B/13B} & 7B W2 & 7B C4 & 13B W2 & 13B C4 \\
\midrule
2 & ButterflyQuant (scalar W2A16) & --- & 15.40 & 16.61 & 10.24 & 12.48 \\
2 & AWQ (g128) & 2.14/2.14 & $2.2{\times}10^{5}$ & $1.7{\times}10^{5}$ & $1.2{\times}10^{5}$ & $9.4{\times}10^{4}$ \\
2 & GPTQ (g128) & 2.14/2.14 & 36.77 & 33.70 & 28.14 & 20.97 \\
2 & OmniQuant (g128) & 2.14/2.14 & 11.06 & 15.02 & 8.26 & 11.05 \\
2 & QuIP\# (RHT) & 2.00/2.00 & 8.95 & 11.22 & 6.52 & 8.32 \\
\rowcolor[HTML]{F2F2F2}
2 & HARP & 2.11/2.05 & \textbf{7.85} & \textbf{9.68} & \textbf{6.13} & \textbf{7.87} \\
\midrule
3 & AWQ (g128) & 3.15/3.15 & 6.24 & 7.84 & 5.32 & 6.94 \\
3 & GPTQ (g128) & 3.15/3.15 & 6.29 & 7.89 & 5.42 & 7.00 \\
3 & OmniQuant (g128) & 3.15/3.15 & 6.03 & 7.75 & 5.28 & 6.98 \\
3 & QuIP\# (RHT) & 3.00/3.00 & 6.00 & 7.60 & 5.23 & 6.85 \\
\rowcolor[HTML]{F2F2F2}
3 & HARP & 3.11/3.05 & \textbf{5.89} & \textbf{7.46} & \textbf{5.15} & \textbf{6.74} \\
\midrule
4 & AWQ (g128) & 4.16/4.16 & 5.62 & 7.13 & 4.97 & 6.56 \\
4 & GPTQ (g128) & 4.16/4.16 & 5.61 & 7.12 & 4.98 & 6.56 \\
4 & OmniQuant (g128) & 4.16/4.16 & \textbf{5.58} & 7.12 & \textbf{4.95} & 6.56 \\
4 & SpinQuant (weight-only ablation) & --- & 5.60 & --- & 5.00 & --- \\
4 & QuIP\# (RHT) & 4.00/4.00 & 5.64 & 7.16 & 5.01 & 6.60 \\
\rowcolor[HTML]{F2F2F2}
4 & HARP & 4.11/4.05 & 5.59 & \textbf{7.10} & 4.96 & \textbf{6.55} \\
\bottomrule
\end{tabular}
\end{table*}

\section{Evaluation uncertainty}
\label{app:evaluation_uncertainty}

For a fixed quantized checkpoint, perplexity and multiple-choice evaluation are deterministic, no stochastic decoding is used. The uncertainty below is the standard error over evaluation examples, not variance across independently re-quantized checkpoints. Calibration data, Rademacher signs, and random seeds are fixed in the controlled RHT/HARP comparison.

\begin{table}[ht]
\centering
\caption{Two-bit Llama~2 perplexity with evaluation standard errors.}
\label{tab:ppl_standard_errors}
\small
\resizebox{\columnwidth}{!}{
\begin{tabular}{@{}lcccc@{}}
\toprule
Model & RHT W2 & HARP W2 & RHT C4 & HARP C4 \\
\midrule
7B  & $8.22\pm0.22$ & $\mathbf{7.23\pm0.19}$ & $10.86\pm0.37$ & $\mathbf{9.49\pm0.35}$ \\
13B & $6.05\pm0.16$ & $\mathbf{5.71\pm0.14}$ & $8.06\pm0.27$ & $\mathbf{7.63\pm0.26}$ \\
70B & $4.16\pm0.10$ & $\mathbf{4.01\pm0.10}$ & $6.01\pm0.16$ & $\mathbf{5.81\pm0.14}$ \\
\bottomrule
\end{tabular}}
\end{table}

\begin{table*}[ht]
\centering
\caption{Two-bit Llama~2 zero-shot accuracy (percentage points, mean $\pm$ evaluation standard error).}
\label{tab:zeroshot_standard_errors}
\small
\begin{tabular}{@{}llcccc@{}}
\toprule
Model & Method & ARC-C & ARC-E & PIQA & WinoGrande \\
\midrule
7B & RHT  & $29.7\pm1.34$ & $56.7\pm1.02$ & $70.8\pm1.06$ & $62.1\pm1.36$ \\
7B & HARP & $\mathbf{33.0\pm1.37}$ & $\mathbf{63.7\pm0.99}$ & $\mathbf{72.4\pm1.04}$ & $\mathbf{62.9\pm1.37}$ \\
13B & RHT  & $33.8\pm1.38$ & $65.1\pm0.98$ & $74.4\pm1.02$ & $64.3\pm1.35$ \\
13B & HARP & $\mathbf{36.4\pm1.41}$ & $\mathbf{67.3\pm0.96}$ & $\mathbf{75.8\pm1.00}$ & $\mathbf{67.2\pm1.32}$ \\
70B & RHT  & $47.4\pm1.46$ & $76.9\pm0.85$ & $79.5\pm0.94$ & $75.0\pm1.22$ \\
70B & HARP & $\mathbf{48.5\pm1.46}$ & $\mathbf{77.8\pm0.87}$ & $\mathbf{79.9\pm0.94}$ & $\mathbf{75.5\pm1.21}$ \\
\bottomrule
\end{tabular}
\end{table*}

The perplexity improvement is consistent on both corpora and every scale. All listed 2-bit zero-shot means are non-decreasing, although several smaller task differences overlap benchmark standard error. A multi-seed re-quantization study would measure sensitivity to calibration/sign initialization rather than evaluation uncertainty and is separate from the fixed-backend comparison here.
Tables~\ref{tab:ppl_standard_errors} and~\ref{tab:zeroshot_standard_errors} report these standard errors.

\section{Calibration-time cost}
\label{app:calibration_time}

HARP fits the fused $QKV$, output, fused up/gate, and down-projection groups independently within each Transformer block. Each group receives 1200 Adam updates. The refresh interval $k$ changes only how frequently the stopped-gradient codebook target $Q(\widetilde W)$ is recomputed: 1200, 600, and 300 evaluations per group for $k=1,2,4$, respectively. Table~\ref{tab:calibration_scaling} reports calibration and QuIP\# quantization time across model scales.

\begin{table*}[t]
\centering
\caption{Calibration and quantization cost for 2-bit Llama~2 on H100 GPUs. Full-model GPU-hours exclude Hessian generation, model export, and evaluation.}
\label{tab:calibration_scaling}
\small
\setlength{\tabcolsep}{0.08cm}
\begin{tabular}{@{}llrrrrr@{}}
\toprule
Model & Method/$k$ & \makecell{Target evals\\per group} & \makecell{Block\\time} & \makecell{Full-model\\GPU-hours} & \makecell{Peak HARP\\VRAM} & W2 PPL $\downarrow$ \\
\midrule
\multirow{4}{*}{7B}
& RHT & --- & $\sim80$s & $\sim0.7$h & --- & 8.22 \\
& HARP/$1$ & 1200 & $\sim1100$s & $\sim9.7$h & 12 GB & 7.235 \\
& HARP/$2$ & 600 & $\sim650$s & $\sim5.7$h & 12 GB & 7.237 \\
& HARP/$4$ & 300 & $\sim450$s & $\sim4.0$h & 12 GB & 7.25 \\
\midrule
\multirow{4}{*}{13B}
& RHT & --- & $\sim140$s & $\sim1.5$h & --- & 6.05 \\
& HARP/$1$ & 1200 & $\sim2000$s & $\sim22$h & 22 GB & 5.71 \\
& HARP/$2$ & 600 & $\sim1300$s & $\sim14$h & 22 GB & 5.68 \\
& HARP/$4$ & 300 & $\sim850$s & $\sim9$h & 22 GB & 5.75 \\
\midrule
\multirow{4}{*}{70B}
& RHT & --- & $\sim270$s & $\sim6$h & --- & 4.16 \\
& HARP/$1$ & 1200 & $\sim3650$s & $\sim80$h & 42--50 GB & 4.01 \\
& HARP/$2$ & 600 & $\sim2200$s & $\sim48$h & 42--50 GB & 4.03 \\
& HARP/$4$ & 300 & $\sim1400$s & $\sim31$h & 42--50 GB & 4.07 \\
\bottomrule
\end{tabular}
\end{table*}

Refreshing every two steps reduces HARP GPU-hours by approximately $36$--$41\%$ relative to $k=1$, with no systematic quality loss in these runs. Refreshing every four steps reduces cost by approximately $59$--$61\%$ while retaining an improvement over RHT at every scale. At 70B, these settings correspond to approximately 3.3, 2.0, and 1.3 days on one H100 for $k=1,2,4$. Independent blocks can be scheduled across GPUs.

Peak memory is controlled by \texttt{harp\_chunk\_size}: smaller chunks reduce VRAM at a modest runtime cost. Disabling $\mathcal{R}_{\mathrm{bd}}$ entirely saves approximately $20$--$30\%$ VRAM and changes 7B WikiText2 PPL from 7.23 to 7.31. For 70B, it is disabled only for the largest down projection, keeping peak usage below 50 GB. Figure~\ref{fig:calibration_tradeoff} gives the finer 7B refresh sweep, including $k=3,5,6$.

\begin{figure}[ht]
\centering
\includegraphics[width=0.9\linewidth]{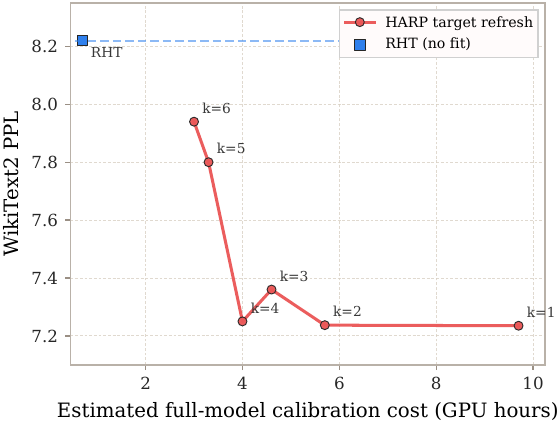}
\caption{Calibration cost--quality trade-off for the detailed Llama~2 7B target-refresh sweep.}
\label{fig:calibration_tradeoff}
\end{figure}

\paragraph{Hessian generation.}
Table~\ref{tab:calibration_scaling} excludes generation of empirical second-moment statistics. For Llama~2 we reuse the released QuIP\# statistics, whose original generation time was not reported. The same statistics are shared by RHT and HARP and reused across bitwidths and refresh settings. As a measured from-scratch reference, generating Qwen3-8B statistics from 6144 sequences at context length $4096$ took approximately 3.1 GPU-hours, or about 5--6 minutes per Transformer layer.

\section{Stored model size}
\label{app:model_size}

Table~\ref{tab:model_size} reports stored model size in GB for the fixed-RHT baseline, HARP with floating-point parameters, and HARP with int8 parameter storage.
All sizes include the quantized model and processor metadata.

\begin{table}[ht]
\centering
\caption{Stored model size in GB.}
\label{tab:model_size}
\small\sc
\setlength{\tabcolsep}{0.12cm}
\renewcommand{\arraystretch}{1.06}
\begin{tabular}{@{}clccc@{}}
\toprule
Bits & Model & RHT & HARP & HARP-int8 \\
\midrule
2 & Llama~3.2 1B & 0.77 & 0.82 & 0.78 \\
\rowcolor[HTML]{F2F2F2}
2 & Llama~3.2 3B & 1.50 & 1.60 & 1.52 \\
2 & Llama~2 7B & 2.15 & 2.41 & 2.20 \\
\rowcolor[HTML]{F2F2F2}
2 & Llama~2 13B & 3.84 & 4.12 & 3.89 \\
2 & Llama~2 70B & 18.18 & 19.21 & 18.38 \\
\midrule
3 & Llama~3.2 1B & 0.90 & 0.95 & 0.91 \\
\rowcolor[HTML]{F2F2F2}
3 & Llama~3.2 3B & 1.85 & 1.96 & 1.87 \\
3 & Llama~2 7B & 2.96 & 3.22 & 3.01 \\
\rowcolor[HTML]{F2F2F2}
3 & Llama~2 13B & 5.43 & 5.70 & 5.48 \\
3 & Llama~2 70B & 26.75 & 27.78 & 26.95 \\
\midrule
4 & Llama~3.2 1B & 1.02 & 1.07 & 1.03 \\
\rowcolor[HTML]{F2F2F2}
4 & Llama~3.2 3B & 2.21 & 2.31 & 2.23 \\
4 & Llama~2 7B & 3.77 & 4.03 & 3.82 \\
\rowcolor[HTML]{F2F2F2}
4 & Llama~2 13B & 7.01 & 7.29 & 7.07 \\
4 & Llama~2 70B & 35.32 & 36.35 & 35.52 \\
\bottomrule
\end{tabular}
\end{table}

\section{Incoherence and quantizer-alignment diagnostics}
\label{app:incoherence_diag}

The purpose of incoherence processing is to change the coordinates seen by the quantizer without changing the full-precision function. We complement the aggregate analysis in Table~\ref{tab:reconstruction_summary} with codebook-shell and curvature diagnostics over the same 128 Llama~2 7B matrices.

\subsection{E8P shell alignment}
E8P quantizes contiguous 8-dimensional source blocks with a bounded codeword radius. Table~\ref{tab:selected_codebook_alignment} reports representative matrices where HARP reduces both direct block distortion and the population outside this radius.

\begin{table}[ht]
\centering
\caption{Selected E8P codebook-alignment diagnostics for 2-bit Llama~2 7B (RHT $\rightarrow$ HARP).}
\label{tab:selected_codebook_alignment}
\small
\resizebox{\columnwidth}{!}{
\begin{tabular}{@{}lcc@{}}
\toprule
Matrix & Mean distortion & Above max norm \\
\midrule
Layer-0 $o$ & $6.90\rightarrow\mathbf{2.50}$ & $39.0\%\rightarrow\mathbf{11.0\%}$ \\
Layer-1 up & $1.14\rightarrow\mathbf{0.86}$ & $3.75\%\rightarrow\mathbf{0.29\%}$ \\
Layer-0 down & $1.15\rightarrow\mathbf{0.90}$ & $3.70\%\rightarrow\mathbf{0.26\%}$ \\
\bottomrule
\end{tabular}}
\end{table}

\begin{figure*}[t]
\centering
\includegraphics[width=0.92\textwidth]{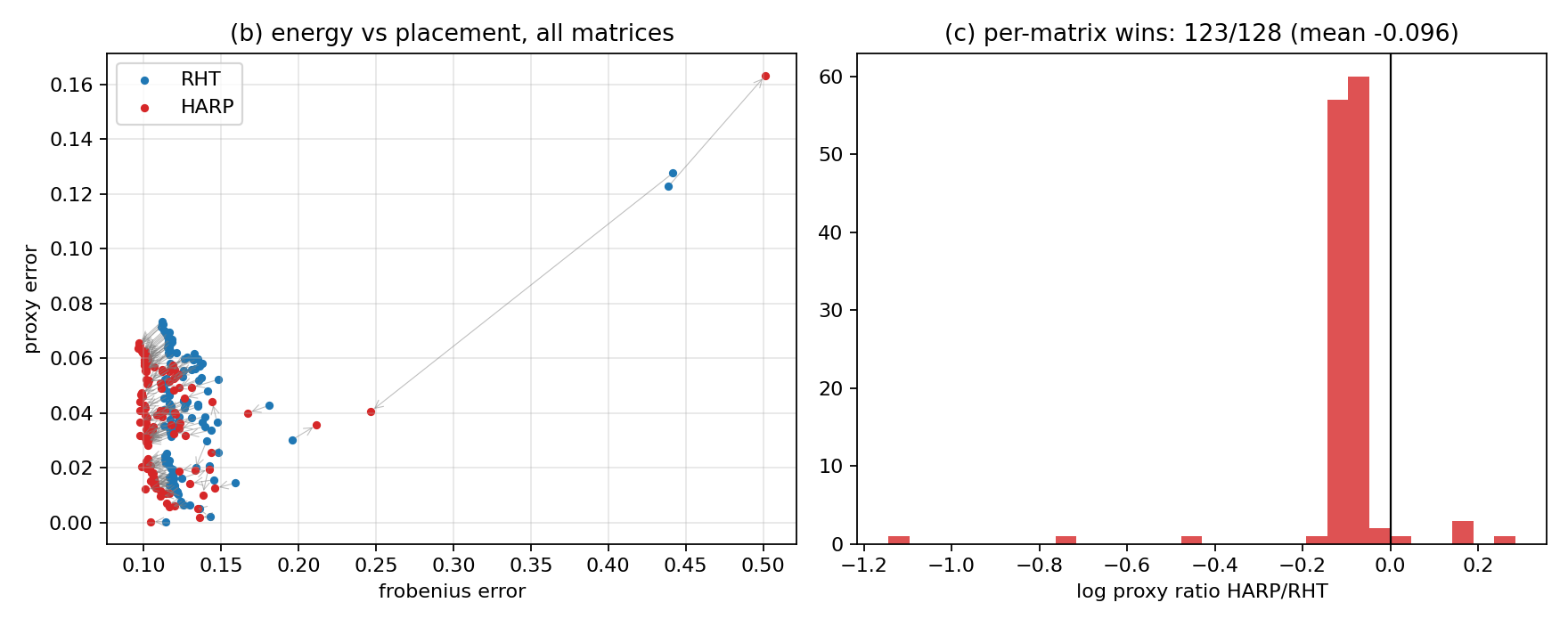}
\caption{Relative Frobenius error and Hessian-weighted proxy error over all 128 matrices, with per-matrix HARP/RHT proxy ratios.}
\label{fig:error_decoupling}
\end{figure*}

Figure~\ref{fig:error_decoupling} shows that HARP generally lowers Hessian-weighted error even when changes in raw error energy are modest: it improves 123 of 128 matrices, with mean log proxy ratio $-0.096$. This pattern is consistent with a learned basis that improves codebook fit and the interaction between residual error and layer curvature.

\begin{figure*}[t]
\centering
\includegraphics[width=0.96\textwidth]{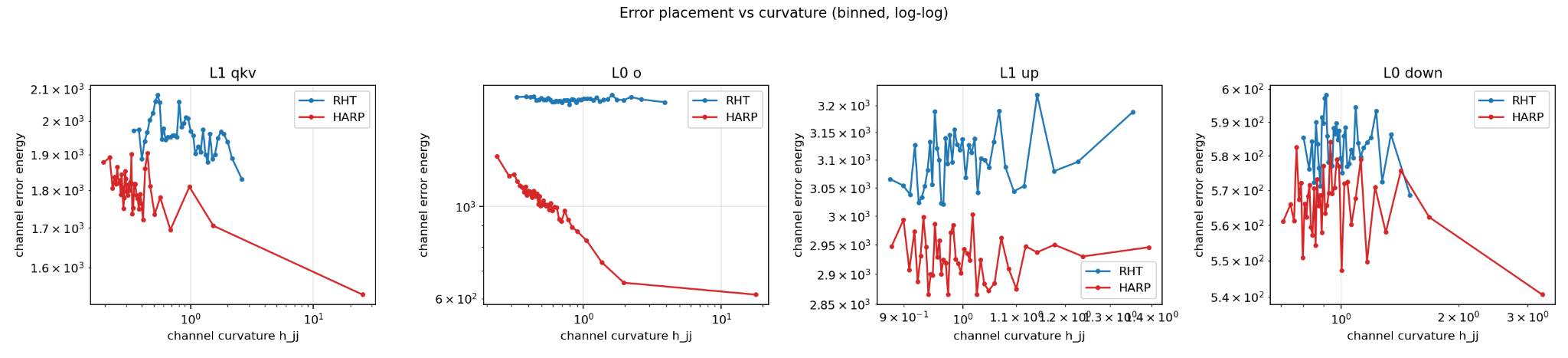}
\caption{Binned channel-error energy versus diagonal curvature for four representative quantized matrices.}
\label{fig:error_vs_curvature}
\end{figure*}

Figure~\ref{fig:error_vs_curvature} provides representative channel-level views. HARP lowers error energy over much of the curvature range, especially for the layer-0 output projection. These plots are diagnostic rather than a complete causal account. Residual directions and cross-layer error accumulation remain important topics for refining the objective.

\subsection{Generic and block-aligned incoherence}
We additionally measure pre- and post-quantization weight incoherence, off-block Hessian energy, diagonal Hessian-weighted distortion, and classical Hessian incoherence. Table~\ref{tab:incoherence_diag} reports these diagnostics. Lower values are better.

\begin{table}[ht]
\centering
\caption{Incoherence and quantizer-alignment diagnostics on Llama~2 7B, 2-bit HARP. ``HARP better'' counts evaluated modules where HARP improves over RHT.}
\vspace{-4pt}
\label{tab:incoherence_diag}
\small\sc
\setlength{\tabcolsep}{0.045cm}
\renewcommand{\arraystretch}{1.06}
\resizebox{\columnwidth}{!}{
\begin{tabular}{@{}lrrrr@{}}
\toprule
Metric & RHT & HARP & \makecell[c]{$\Delta$} & \makecell{HARP\\better} \\
\midrule
$\mu_W(\widetilde W_{\mathrm{pre}})$
& 5.7712 & \textbf{5.0972} & \textbf{$-0.6741$} & 118 / 128 \\
\rowcolor[HTML]{F2F2F2}
$\mu_W(\widehat{\widetilde W})$
& 3.0159 & \textbf{3.0049} & \textbf{$-0.0110$} & 125 / 128 \\
$\mathrm{OffBlk}(\widetilde H)$
& 0.9634 & \textbf{0.9237} & \textbf{$-0.0397$} & 103 / 128 \\
\rowcolor[HTML]{F2F2F2}
$\mathcal L_{\mathrm{diag}}(\widehat{\widetilde W})$
& $2.538{\times}10^{-4}$
& $\bm{2.493{\times}10^{-4}}$
& $-4.57{\times}10^{-6}$
& 126 / 128 \\
$\mu_H(\widetilde H)$
& \textbf{6.5997} & 11.2700 & $+4.6703$ & 13 / 96 \\
\bottomrule
\end{tabular}}
\end{table}

HARP improves pre-quantization weight incoherence, quantized-weight incoherence, off-block Hessian energy, and diagonal Hessian-weighted distortion in almost all evaluated modules. It does not improve the classical Hessian incoherence score $\mu_H$, for which RHT is better. This is consistent with the optimization target: HARP is not trained to minimize every generic incoherence measure, but to find an orthogonal basis favorable to the deployed blockwise quantizer. The combined evidence indicates that HARP better matches the E8P blocks and the layer's second-order geometry, explaining why quantizer-aligned diagnostics improve even when classical Hessian incoherence does not.

\section{Additional experiments}
\label{app:extra_exp}
\subsection{Kronecker compatibility variant}
\label{app:kron_ablation}

The Kronecker variant follows QuIP\#'s non-power-of-two Hadamard convention and reduces processor storage, while Mixed-Radix learns over the full dimension. Table~\ref{tab:kron_results} reports the models on which both variants were evaluated.

\begin{table*}[ht]
\centering
\caption{Two-bit Kronecker compatibility results and Mixed-Radix comparison.}
\label{tab:kron_results}
\small
\resizebox{0.75\textwidth}{!}{
\begin{tabular}{@{}lccc ccc ccc@{}}
\toprule
& \multicolumn{3}{c}{Llama~3.2 1B} & \multicolumn{3}{c}{Llama~3.2 3B} & \multicolumn{3}{c}{Llama~2 7B} \\
\cmidrule(lr){2-4}\cmidrule(lr){5-7}\cmidrule(lr){8-10}
Method & BPP & W2 & C4 & BPP & W2 & C4 & BPP & W2 & C4 \\
\midrule
RHT & 2.00 & 26.27 & 25.24 & 2.00 & 16.59 & 15.88 & 2.00 & 8.22 & 10.86 \\
HARP (Kronecker) & 2.13 & 23.36 & 22.88 & 2.06 & 15.97 & 15.29 & 2.02 & 7.61 & 9.98 \\
\hspace{6pt}+ int8 params & 2.07 & 23.38 & 22.90 & 2.03 & 16.06 & 15.30 & 2.01 & 7.67 & 10.06 \\
\rowcolor[HTML]{F2F2F2}
HARP (Mixed-Radix) & 2.14 & \textbf{22.30} & \textbf{22.57} & 2.10 & \textbf{15.02} & \textbf{14.77} & 2.11 & \textbf{7.23} & \textbf{9.49} \\
\bottomrule
\end{tabular}}
\end{table*}

The Kronecker construction improves over RHT at lower overhead, while Mixed-Radix is consistently more accurate and is therefore the primary variant used in the complete cross-scale study.

\subsection{Ablation: stage radix as a quality/overhead knob}
\label{sec:experiments:radix_ablation}

The stage radix controls both expressivity and overhead.
For fixed dimension $d$, larger radix reduces the number of stride stages (e.g., for $d=4096$, the stage count is
$\log_b d$), which can reduce reshape/transpose and kernel-launch overhead in practice.
However, larger blocks have more degrees of freedom and therefore increase parameter overhead, which is reflected
in BPP. Table~\ref{tab:ablate_radix} shows this trade-off on Llama~2 7B.

\begin{table}[ht]
\centering
\caption{Effect of the preferred radix $b$ on Llama~2 7B (2-bit PTQ)}
\vspace{-4pt}
\label{tab:ablate_radix}
\small\sc
\tabcolsep=0.11cm
\renewcommand{\arraystretch}{1.08}
\begin{tabular}{@{}ccccc@{}}
\toprule
$b$ & Stages (4096) & BPP & W2 PPL $\downarrow$ & C4 PPL $\downarrow$ \\
\midrule
2  & 12 & 2.08 & 7.38 & 9.73 \\
\rowcolor[HTML]{F2F2F2}
4  & 6  & 2.09 & 7.36 & 9.65 \\
8  & 4  & 2.11 & 7.23 & 9.49 \\
\rowcolor[HTML]{F2F2F2}
16 & 3  & 2.15 & \textbf{7.17} & \textbf{9.35} \\
\bottomrule
\end{tabular}
\end{table}

\subsection{Ablation: base mixer choice}
\label{app:mixer_choice}

Table~\ref{tab:mixer_choice} compares using an identity base mixer against the default Hadamard/QR base mixers.
Hadamard preconditioning provides a stronger initialization and slightly improves the final perplexity.

\begin{table}[t]
\centering
\caption{Base mixer choice for HARP blocks on Llama~2 7B, 2-bit PTQ.}
\vspace{-4pt}
\label{tab:mixer_choice}
\small\sc
\setlength{\tabcolsep}{0.12cm}
\renewcommand{\arraystretch}{1.06}
\begin{tabular}{@{}lcc@{}}
\toprule
Mixer & W2 PPL $\downarrow$ & C4 PPL $\downarrow$ \\
\midrule
Identity ($G_b=I$) & 7.29 & 9.61 \\
\rowcolor[HTML]{F2F2F2}
Hadamard+QR & \textbf{7.23} & \textbf{9.49} \\
\bottomrule
\end{tabular}
\end{table}

\subsection{Ablation: off-block regularization}
\label{app:hbd_ablation}

The coefficient $\lambda_{\mathrm{bd}}$ controls how strongly the input-side rotation is encouraged to localize curvature within the quantizer's contiguous blocks. Table~\ref{tab:hbd_ablation} sweeps $\lambda_{\mathrm{bd}}$ on Llama~2 7B at 2 bits.

\begin{table}[ht]
\centering
\caption{Off-block regularizer ablation on Llama~2 7B at 2 bits.}
\label{tab:hbd_ablation}
\small
\begin{tabular}{@{}ccc@{}}
\toprule
$\lambda_{\mathrm{bd}}$ & W2 PPL $\downarrow$ & C4 PPL $\downarrow$ \\
\midrule
0.0 & 7.31 & 9.56 \\
0.1 & \textbf{7.23} & \textbf{9.49} \\
0.2 & 7.33 & 9.64 \\
0.3 & 7.32 & 9.61 \\
0.5 & 7.42 & 9.71 \\
\bottomrule
\end{tabular}
\end{table}

A mild penalty improves both datasets, while larger values increasingly constrain $V$ and reduce quality. The modest difference between $\lambda_{\mathrm{bd}}=0$ and $0.1$ shows that block localization is useful but is not the sole source of HARP's gain.

\section{Additional scaling plots}
\label{app:scaling_plots}

\begin{figure*}[t]
\centering

\begin{subfigure}[t]{0.42\textwidth}
  \centering
  \includegraphics[width=\linewidth]{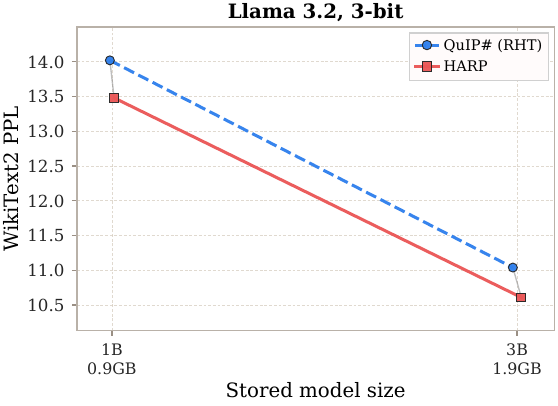}
  \caption{Llama~3.2, 3-bit.}
  \label{fig:scaling_w2_3bit_small}
\end{subfigure}
\hspace{16pt}
\begin{subfigure}[t]{0.42\textwidth}
  \centering
  \includegraphics[width=\linewidth]{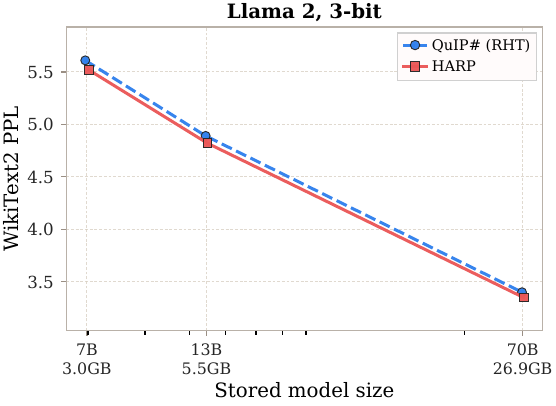}
  \caption{Llama~2, 3-bit.}
  \label{fig:scaling_w2_3bit_large}
\end{subfigure}

\vspace{0.6em}

\begin{subfigure}[t]{0.42\textwidth}
  \centering
  \includegraphics[width=\linewidth]{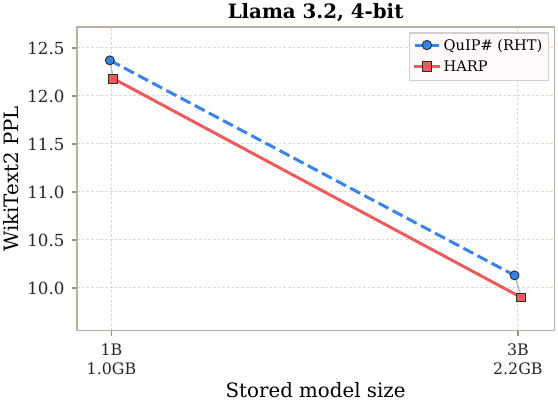}
  \caption{Llama~3.2, 4-bit.}
  \label{fig:scaling_w2_4bit_small}
\end{subfigure}
\hspace{16pt}
\begin{subfigure}[t]{0.42\textwidth}
  \centering
  \includegraphics[width=\linewidth]{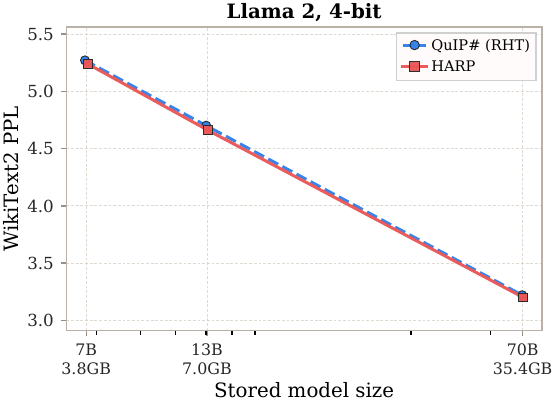}
  \caption{Llama~2, 4-bit.}
  \label{fig:scaling_w2_4bit_large}
\end{subfigure}

\caption{WikiText2 quality--size scaling at 3 and 4 bits. HARP uses int8 parameter storage. Llama~3.2 models use context length $8192$, Llama~2 models use context length $4096$.}
\label{fig:scaling_w2_3bit_4bit}
\end{figure*}

Figure~\ref{fig:scaling_w2_3bit_4bit} shows WikiText2 quality--size scaling at 3 and 4 bits.
The absolute gains are smaller than at 2 bits because the RHT baseline is already closer to FP16, but HARP continues to shift the curve downward at nearly unchanged storage when int8 parameter storage is used.

\section{Compatibility with QuIP\# fine-tuning}
\label{app:ft_ablation}

The main experiments compare no-finetuning quantization to isolate the effect of the incoherence processor.
QuIP\# also reports results with additional fine-tuning stages.
We distinguish two such stages.

\emph{FT-quant} denotes QuIP\#'s fine-tuning-during-quantization stage.
Modules within each Transformer block are quantized in a fixed order, starting with the joint $QKV$ projection.
After the current module or module group is quantized, optimization updates the channel scales for modules that have already been quantized and the full weights for modules that have not yet been quantized.
Thus, FT-quant locally adapts the partially quantized block while quantization proceeds through the block.

\emph{E2E FT} denotes the subsequent end-to-end fine-tuning stage after all modules have been quantized.
At this stage, the procedure fine-tunes quantization scales together with the remaining unquantized parameters in the network, such as normalization parameters and the language-model head.

These stages improve final quality, but they introduce additional choices such as trainable parameter sets, learning rates, schedules, and data budgets.
For this reason, our main tables omit them and isolate the RHT-to-HARP processor change.

HARP is compatible with these stages because it only changes the orthogonal preprocessing used before quantization.
As a preliminary check, we enable FT-quant on Llama~2 7B at context length $4096$.
Table~\ref{tab:ft_ablation} reports the result.

\begin{table}[ht]
\centering
\vspace{-2mm}
\caption{Fine-tuning compatibility on Llama~2 7B, 2-bit, context length $4096$.
HARP uses Mixed-Radix.}
\label{tab:ft_ablation}
\small\sc
\setlength{\tabcolsep}{0.10cm}
\renewcommand{\arraystretch}{1.06}
\begin{tabular}{@{}lcc@{}}
\toprule
Method & W2 PPL $\downarrow$ & C4 PPL $\downarrow$ \\
\midrule
QuIP\# + RHT + FT-quant only & 6.44 & 8.30 \\
QuIP\# + RHT + FT-quant + E2E FT & 6.19 & 8.16 \\
\rowcolor[HTML]{F2F2F2}
HARP + FT-quant only & \textbf{6.16} & \textbf{8.08} \\
\bottomrule
\end{tabular}
\end{table}

Even with only FT-quant, HARP improves over the corresponding RHT setting.
It also slightly improves over the fully tuned RHT result that uses both FT-quant and E2E FT.
We do not include full HARP + E2E FT experiments in the main table because the goal is to isolate the processor and avoid conflating it with fine-tuning schedule choices.

\section{Experimental Reproducibility Checklist}
\label{app:repro}
\subsection{Hardware}
\label{app:repro:hardware}

Quantization and calibration experiments use NVIDIA H100 GPUs unless otherwise stated. Inference latency is measured on an NVIDIA GeForce RTX 5080 with 16GB VRAM. HARP performs 1200 Adam updates for each of four quantized module groups per Transformer block. Table~\ref{tab:calibration_scaling} reports block times, full-model GPU-hours, and peak VRAM through Llama~2 70B. Hessian generation is a separate reusable preprocessing stage and is excluded from those totals.

\subsection{Evaluation}
Reported perplexity and multiple-choice values are deterministic for a fixed quantized checkpoint. Appendix~\ref{app:evaluation_uncertainty} reports standard errors over evaluation examples for the principal 2-bit results. These standard errors quantify benchmark sampling uncertainty, they do not measure variation across independently regenerated calibration sets, signs, or quantized checkpoints.

\subsection{HARP fitting hyperparameters}
\label{app:repro:harp_hparams}

Unless stated otherwise, we fit one pair of HARP processors $(U,V)$ per quantized module group using $S=1200$ Adam steps.
We use stride ordering (Algorithm~\ref{alg:harp_transform}) with a single pass ($P=1$) and the default Mixed-Radix schedule
constructed by Algorithm~\ref{alg:greedy_schedule} with preferred radix $b_{\mathrm{base}}=8$ and maximum radix $b_{\max}=8$.

\paragraph{QuIP\# hyperparameters.}
We keep the QuIP\# backend fixed throughout: codebook family and training, block sizes, quantization scales, and all
QuIP\# solver/rounding hyperparameters follow the settings in the original QuIP\# paper~\citep{quip_sharp}. We do not
retune QuIP\#-specific hyperparameters for HARP, the only change is replacing the fixed Hadamard mixer with the learned
HARP processor (and optionally enabling the Kronecker fallback where stated).

\paragraph{Calibration statistics.}
For the Llama experiments, we reuse the precomputed Hessian/second-moment statistics from the corresponding QuIP\# pipeline. Following the QuIP\# protocol, these statistics use 6144 sequences sampled from RedPajama 1T \citep{quip_sharp}. For the Qwen3-8B transfer experiment, we generate statistics from 6144 sequences at context length $4096$. This preprocessing takes approximately 3.1 GPU-hours in our setup. Hessian generation is a one-time stage whose outputs can be reused across quantization methods, bitwidths, hyperparameter sweeps, and target-refresh settings.

\paragraph{Optimization.}
We use learning rates $\eta_U=\eta_V=3\times 10^{-2}$ for the $U$ and $V$ parameters.
We initialize rotation parameters at $\Theta=0$, so that $Q_{t,c}(0)=I$
and HARP matches Hadamard-family preprocessing at initialization (Section~\ref{sec:method:exact_init}).

\paragraph{Objective and regularization.}
We use the diagonal-weighted proxy objective from Eq.~\eqref{eq:diag_proxy} and the block-diagonalization regularizer
from Eq.~\eqref{eq:hbd} with $\lambda_{\mathrm{bd}}=0.1$ and block size $g=8$. For Llama~2 70B down projection, we disable $\mathcal{R}_{\mathrm{bd}}$ to keep peak calibration memory below 50 GB.

\paragraph{Base mixers.}
For power-of-two radices we use Hadamard base mixers, and for non-power-of-two radices we use the QR-based orthogonal
fallback (Section~\ref{sec:method:kernels_init}).

\subsection{Code and artifact availability}
\label{sec:code}

The HARP implementation, QuIP\#/QTIP integrations, deployment kernels, and reproduction scripts are available at \url{https://github.com/brain-lab-research/HARP}; quantized HARP checkpoints are available at  \url{https://huggingface.co/collections/brain-lab/harp-quantized-models}.

\subsection{Artifacts and licenses}
\label{app:artifacts_licenses}

We use publicly released Llama~2, Llama~3.2, and Qwen3 checkpoints under their original model licenses, and evaluate on WikiText2, C4, ARC-Challenge, ARC-Easy, PIQA, and WinoGrande through the standard evaluation harness. We release the software required to reproduce HARP experiments together with quantized checkpoints for research and reproducibility. Derived checkpoints remain subject to the licenses and access terms of their base models. We do not redistribute benchmark or calibration text. The released code is based on the public QuIP\# and QTIP implementations and follows the upstream licenses stated in the repository README.

\section{Implementation details}
\label{app:impl}

This section reports implementation details omitted from the main text: the stride-layout example, the Mixed-Radix schedule construction, the layerwise fitting pseudocode, and parameter packing.

\subsection{Stride-stage implementation}
\label{app:stride_impl_details}

The permutation $P_t$ in Eq.~\eqref{eq:harp_stage} is conceptual.
A stage is implemented by reshaping, transposing, multiplying small blocks, and reshaping back.

\paragraph{Index-view example.}
For $d=16$ and $b=4$, write an index as two base-4 digits
$i=i_0+4i_1$.
The first stage mixes the low-order digit $i_0$ while holding $i_1$ fixed; these are contiguous groups.
The second stage mixes the high-order digit $i_1$ while holding $i_0$ fixed; these are stride-4 groups.
Thus each stage mixes one digit of the Mixed-Radix index representation, as in fast Walsh--Hadamard and Cooley--Tukey-style transforms, but implemented by reshape/transpose rather than by materializing a permutation.

\subsection{Stage schedules and Mixed-Radix construction}
\label{app:greedy_schedule}

Algorithm~\ref{alg:greedy_schedule} gives the schedule construction used
in all experiments. It peels off as many copies of the preferred radix
$b_{\mathrm{base}}$ as divide $d$, factors the remainder with the largest
admissible factor at most $b_{\max}$, and keeps any leftover prime as a
final stage. We use $b_{\mathrm{base}}=b_{\max}=8$; the schedule is
deterministic in $d$ and needs no padding. For example, $4096$ gives
$(8,8,8,8)$ and $5120$ gives $(8,8,8,5,2)$, with the QR fallback of
Section~\ref{sec:method:kernels_init} used only for the radix-$5$ stage.

\begin{algorithm}[ht]
\caption{Greedy Mixed-Radix schedule (preferred radix first)}
\label{alg:greedy_schedule}
\small
\begin{algorithmic}[1]
\REQUIRE Dimension $d\ge 2$, preferred radix $b_{\mathrm{base}}$ (default $8$), maximum radix $b_{\max}$ (default $8$)
\ENSURE A list of stage radices $\bm{b}=(b_0,\dots,b_{m-1})$ such that $\prod_t b_t=d$
\STATE $\bm{b}\leftarrow [\;]$, \quad $r \leftarrow d$
\WHILE{$r \bmod b_{\mathrm{base}} = 0$}
  \STATE append $b_{\mathrm{base}}$ to $\bm{b}$
  \STATE $r \leftarrow r / b_{\mathrm{base}}$
\ENDWHILE
\FOR{$f=\min(b_{\max},r)$ down to $2$}
  \WHILE{$r \bmod f = 0$}
    \STATE append $f$ to $\bm{b}$
    \STATE $r \leftarrow r / f$
  \ENDWHILE
\ENDFOR
\IF{$r \neq 1$}
  \STATE append $r$ to $\bm{b}$
\ENDIF
\STATE \textbf{return} $\bm{b}$
\end{algorithmic}
\end{algorithm}

\subsection{Layerwise fitting procedure}
\label{app:harp_fit_algorithm}
The full layerwise fitting procedure is given in Algorithm~\ref{alg:harp_fit}. It iterates over Adam steps, refreshing the quantized codebook target every $k$ steps to reduce calibration cost.

\begin{algorithm}[ht]
\caption{Layerwise fitting of HARP for a linear layer}
\label{alg:harp_fit}
\small
\begin{algorithmic}[1]
\REQUIRE Weight $W\in\mathbb{R}^{d_{\mathrm{out}}\times d_{\mathrm{in}}}$, second moment $H\in\mathbb{R}^{d_{\mathrm{in}}\times d_{\mathrm{in}}}$
\REQUIRE Blockwise codebook quantizer $Q(\cdot)$; HARP initialization $\Theta_U=\Theta_V=0$
\REQUIRE Adam step size $\eta$, steps $S$, block size $g$, regularizer $\lambda_{\mathrm{bd}}$, target refresh interval $k$ (default $k=1$)
\STATE Initialize $U(\Theta_U)\in O(d_{\mathrm{out}})$ and $V(\Theta_V)\in O(d_{\mathrm{in}})$ as HARP processors
\FOR{$s=1$ to $S$}
    \STATE Compute rotated weights $\widetilde{W} \leftarrow U^\top W V$
    \STATE Compute rotated second moment $\widetilde{H} \leftarrow V^\top H V$
    \STATE $\bar{w}_j \leftarrow |\widetilde{H}_{jj}| / \mathrm{mean}(|\widetilde{H}_{jj}|)$ and stopgrad$(\bar{w})$
    \IF{$s \bmod k = 1$ or $k=1$}
      \STATE $\widehat{\widetilde{W}} \leftarrow \mathrm{stopgrad}(Q(\widetilde{W}))$
    \ENDIF
    \STATE $\Delta \leftarrow \widetilde{W} - \widehat{\widetilde{W}}$
    \STATE $\mathcal{L}_{\mathrm{diag}} \leftarrow \frac{1}{d_{\mathrm{out}}d_{\mathrm{in}}}\sum_{i,j} \Delta_{ij}^2 \,\bar{w}_j$
    \STATE Compute $\mathcal{R}_{\mathrm{bd}}$ from Eq.~\eqref{eq:hbd}
    \STATE $\mathcal{L}_{\mathrm{fit}} \leftarrow \mathcal{L}_{\mathrm{diag}} + \lambda_{\mathrm{bd}}\mathcal{R}_{\mathrm{bd}}$
    \STATE Update $\Theta_U,\Theta_V$ with Adam using $\nabla \mathcal{L}_{\mathrm{fit}}$
\ENDFOR
\STATE Optionally quantize HARP parameters to int8 as in Section~\ref{sec:method:theta_quant}
\RETURN Fitted processors $U,V$
\end{algorithmic}
\end{algorithm}

\subsection{Int8 storage of HARP parameters}
\label{app:theta_quant_details}

After fitting, we optionally quantize HARP parameters to 8-bit integers for storage efficiency.
For $b_t>2$ stages we store only the strict upper triangle of the skew-symmetric matrix $A(\theta)$, which has $b_t(b_t-1)/2$ entries per block.
For $b_t=2$ stages we store the single Givens angle.
For each block $c$, we compute a per-block scale
\[
s_{t,c}=\max |a_{t,c}|/127
\]
and store
\[
q_{t,c}=\mathrm{round}(a_{t,c}/s_{t,c}).
\]
At runtime we reconstruct $a_{t,c}\approx s_{t,c}q_{t,c}$ and build the corresponding orthogonal block using the Givens or Cayley parameterization.

\section{Detailed Literature Review}
\label{sec:related}
We organize related work into three parts: post-training quantization and its failure modes, transform-based reparameterizations, and structured orthogonal transforms for quantization.

\subsection{Post-Training Quantization of LLMs}
\label{sec:related:ptq}

\paragraph{General compression and PTQ.}
LLM deployment under tight memory and bandwidth budgets has motivated a broad set of compression techniques, including one-shot pruning \citep{sparsegpt, wanda}, quantization-aware training (QAT) \citep{llmqat}, and post-training quantization \citep{survey}.
We focus on PTQ because it operates on a fixed pretrained model with only a small calibration set, which keeps it practical at frontier scale where QAT or full retraining is costly.

\paragraph{Layerwise second-order reconstruction.}
The layerwise objective in Eq.~\eqref{eq:proxy_loss} descends from classical second-order pruning \citep{obd, obs, obs2}.
OBC turns this analysis into an exact greedy solver for joint pruning and quantization \citep{obc}, GPTQ scales it to LLM-sized layers \citep{gptq}, and later work accelerates the underlying sparse recovery \citep{i-obs}.
ZeroQuant pairs fine-grained group-wise quantization with lightweight layerwise distillation \citep{zeroquant}.
More recent rounding procedures correct errors inherited from previously quantized layers \citep{qronos} or target the divergence of the full model output directly \citep{yaqa}.
All of these methods improve the rounding rule under a fixed basis.
HARP is complementary: it changes the basis exposed to any such solver while leaving the rounding rule untouched.

\paragraph{Heavy-tailed statistics and outlier channels.}
Extreme low-bit quantization is dominated by heavy-tailed weight distributions and a small number of high-magnitude outlier channels that inflate per-tensor scales and reduce effective precision~\citep{spinquant,quarot}.
A direct remedy is to separate outliers and preserve them in higher precision: LLM.int8()~\citep{llmint8} routes outlier features to FP16, while SpQR and SqueezeLLM decompose weights into a dense quantized core plus a sparse high-precision correction~\citep{spqr,squeezellm}.
QUIK extends the strategy to both weights and activations for end-to-end low-bit inference~\citep{quik}.
Although effective, these methods introduce irregular memory access patterns and specialized kernel requirements that can negate the throughput advantages.
This motivates approaches that maintain dense, uniform low-bit computation while reshaping layer statistics.

\paragraph{Calibration-based rescaling.}
A practical alternative to explicit outlier separation is to apply lightweight, calibration-guided channel-wise rescalings.
AWQ derives per-channel scales from activation magnitudes to protect sensitive channels under weight-only quantization~\citep{awq}.
SmoothQuant shifts part of the activation dynamic range into the weights via a matched diagonal rescaling, enabling standard integer kernels~\citep{smoothquant}.
OmniQuant learns per-layer scale and clipping adjustments during calibration, though clipping is not a strict change of basis and can alter the function~\citep{omniquant}.
RPTQ instead reorders and clusters activation channels so that similarly ranged channels share quantization parameters~\citep{rptq}.
While practical, such rescaling and reordering cannot mix coordinates and therefore leave correlated outlier subspaces intact, a limitation that grows increasingly severe at extreme bitwidths.
This motivates orthogonal reparameterizations that mix all coordinates while preserving the full-precision model exactly.

\subsection{Transform-Based Reparameterizations}
\label{sec:related:transforms}

\paragraph{Incoherence processing.}
QuIP~\citep{quip} introduces incoherence processing: it applies a structured random orthogonal change of basis so that weight mass and curvature-sensitive directions are spread across coordinates.
QuIP\#~\citep{quip_sharp} replaces the general random orthogonals with the randomized Hadamard transform (RHT), which is exactly orthogonal and admits $\mathcal{O}(d \log d)$ application.
QTIP~\citep{qtip} retains the same preprocessing while advancing the quantizer backend to trellis-coded quantization.
Beyond weight-only PTQ, QuaRot~\citep{quarot} inserts discrete orthogonal rotations into the full Transformer graph to suppress outliers in weights, activations, and the KV cache.
A shared limitation of all these methods is the fixed transform.

\paragraph{Learned model-level rotations.}
SpinQuant~\citep{spinquant} demonstrates that the choice of rotation substantially affects low-bit quality and learns Transformer-invariant rotations (including residual-stream and attention-head rotations) end-to-end against a quantized-network objective.
The learned rotations are model-level objects wired into the Transformer graph, making SpinQuant an important baseline but not a drop-in replacement for the two-sided Hadamard preprocessor inside QuIP\#-style pipelines.
HARP, in contrast, operates at the level of a single linear layer and is fit within the PTQ calibration loop.
It is designed as a reusable processor module for backends that already expose an incoherence-processing step.

\paragraph{Non-orthogonal data-aware transforms.}
WUSH~\citep{wush} derives closed-form near-optimal blockwise transforms for joint weight--activation quantization under RTN AbsMax-scaled quantizers, combining a Hadamard backbone with a calibration-dependent second-moment factor.
The resulting transform is generally non-orthogonal and targets standard RTN-style block quantizers at 4-bit weight/activation precision, a different operating point from HARP's focus on extreme low-bit PTQ (2--4 bits per weight) with vector-quantized backends.
More fundamentally, non-orthogonal transforms alter the full-precision computation and are not drop-in change-of-basis modules. HARP deliberately stays strictly orthogonal so the full-precision model is preserved exactly.

\paragraph{Rate allocation, information-theoretic views.}
WaterSIC~\citep{watersic} analyzes linear-layer quantization through rate--distortion theory and derives a waterfilling allocation across input columns.
Q-Palette~\citep{qpalette} connects near-optimal Gaussian-weight bit allocation with practical fractional-bit quantizers.
Both are complementary to HARP: they modify bit allocation or the quantizer family, whereas HARP modifies the orthogonal basis for a fixed backend.

\paragraph{Structured code families.}
The effectiveness of extreme low-bit PTQ depends heavily on the quantizer code family.
AQLM~\citep{aqlm} achieves extreme compression through additive multi-codebook quantization, and PV-Tuning further improves such extreme-compression pipelines by fine-tuning quantized representations beyond straight-through estimation~\citep{pvtuning}.
GPTVQ and VPTQ demonstrate that Hessian-aware vector quantization can improve the size--accuracy trade-off by quantizing higher-dimensional weight blocks~\citep{gptvq,vptq}.
QuIP\# combines incoherence processing with lattice/codebook vector quantization, while QTIP replaces explicit codebooks with trellis coding to increase the effective quantization dimension~\citep{quip_sharp,qtip}.
As these backends grow more powerful, the choice of coordinate system becomes more consequential: HARP learns this coordinate system, leaving the chosen backend unchanged.

\subsection{Structured Orthogonal Transforms}
\label{sec:related:butterfly}

Butterfly factorizations express a broad family of fast transforms as products of sparse orthogonal stages and make these stages learnable~\citep{butterflyfactorization}.
Structured orthogonal parameterizations of this kind also power parameter-efficient fine-tuning, where orthogonal updates~\citep{oft} are made scalable through butterfly factorization~\citep{boft}.
The closest direction to HARP is the use of learnable structured orthogonal rotations as quantization preprocessors.
ButterflyQuant~\citep{butterflyquant} replaces fixed Hadamard rotations with learnable butterfly transforms parameterized by continuous Givens rotations, preserving orthogonality and $\mathcal{O}(d \log d)$ cost while adapting the rotation to calibration data.
HARP shares the principle that structured orthogonal transforms provide learnable mixing at deployable cost, but differs in target and compatibility constraints.

Specifically, HARP is designed as a drop-in replacement for the two-sided RHT inside QuIP\#-style PTQ pipelines, with four distinguishing properties.
First, HARP initializes exactly to the corresponding RHT preprocessor (up to a fixed permutation convention), so calibration learns a structured refinement around an already-strong baseline rather than from scratch.
Second, HARP uses Hadamard-preconditioned block kernels and supports non-power-of-two dimensions through Mixed-Radix schedules without padding.
Third, HARP optimizes a layerwise objective directly coupled to the deployed blockwise quantizer, including a block-diagonalization regularizer that aligns the learned curvature with the backend's block structure.
Fourth, HARP is compatible with modern vector-quantized backends such as QuIP\# and QTIP, enabling backend-agnostic deployment.

Table~\ref{tab:baseline_2} includes ButterflyQuant's published context-matched W2A16 perplexities as a system-level reference. HARP obtains lower perplexity in that comparison, but the scalar and E8P/LDLQ backends differ. The controlled evidence for the learned processor remains the RHT-to-HARP replacement within the same QuIP\# pipeline.

\paragraph{Summary of positioning.}
Relative to fixed RHT~\citep{quip_sharp,qtip}, HARP preserves fast structured execution while adding per-layer, backend-aware adaptivity.
Relative to model-level learned rotations~\citep{spinquant}, HARP uses a staged structured parameterization with controlled overhead and does not require rewiring the Transformer graph.
Relative to non-orthogonal data-aware transforms~\citep{wush}, HARP stays strictly orthogonal so the full-precision model is preserved exactly.
Rather than proposing a new standalone pipeline or quantizer, HARP treats the incoherence processor as a reusable module insertable into any Hadamard-based backbone.

\section*{Information About Use of AI Assistants}
AI assistants were used to proofread the manuscript, support literature search, and organize the final repository structure. All AI-assisted outputs were reviewed by the authors.

\end{document}